%% file: manuscript.tex
\documentclass{article}

\PassOptionsToPackage{numbers, compress}{natbib}

\usepackage[preprint]{neurips_2026}

\usepackage[utf8]{inputenc} 
\usepackage[T1]{fontenc}    
\usepackage{hyperref}       
\usepackage{url}            
\usepackage{booktabs}       
\usepackage{amsfonts}       
\usepackage{nicefrac}       
\usepackage{microtype}      
\usepackage{xcolor}         
\usepackage{orcidlink}

\usepackage{amsmath}
\usepackage{amssymb}
\usepackage{mathtools}
\usepackage{amsthm}

\theoremstyle{definition}
\newtheorem{theorem}{Theorem}[section]
\newtheorem{proposition}[theorem]{Proposition}
\newtheorem{definition}{Definition}[section]

\usepackage{overpic}
\usepackage{subcaption}
\usepackage[capitalize,noabbrev]{cleveref}
\usepackage{tikz}
\usepackage[dvipsnames]{xcolor}
\usepackage{comment}
\usepackage{array}
\usepackage{arydshln} 
\usepackage{algorithm}
\usepackage{algorithmic}
\usepackage{enumitem}
\usepackage{fullwidth}
\usepackage{wrapfig}

\newcolumntype{H}{>{\setbox0=\hbox\bgroup}c<{\egroup}}

\newcommand{\gr}[1]{\textcolor{gray}{#1}}

\definecolor{windowsblue}{HTML}{0173b2}
\definecolor{amber}{HTML}{d55e00}
\definecolor{palered}{HTML}{d55e00}
\definecolor{greyish}{HTML}{808080}
\definecolor{fadedgreen}{HTML}{029e73}
\definecolor{dustypurple}{HTML}{825f87}
\definecolor{dustyorange}{HTML}{de8f05}
\definecolor{pinka}{HTML}{d42aff}
\definecolor{lightbluea}{HTML}{5fd3bc}
\definecolor{lightblueb}{HTML}{2ad4ff}
\definecolor{lightgreena}{HTML}{abc837}

\title{The Role of Node Features in Graph Pooling}

\author{%
  Jan von Pichowski~\mbox{\orcidlink{0009-0001-1541-6477}} \\
  Center for AI and Data Science \\
  University of W{\"u}rzburg, Germany  \\
  \texttt{jan.pichowski@uni-wuerzburg.de}  \\
  \And
  Al{\v z}beta Hrabo{\v s}ov{\'a}~\mbox{\orcidlink{0009-0006-1350-8454}} \\
  Center for AI and Data Science \\
  University of W{\"u}rzburg, Germany \\
  \texttt{alzbeta.hrabosova}\\ \texttt{@stud-mail.uni-wuerzburg.de}
  \AND
  Ingo Scholtes~\mbox{\orcidlink{0000-0003-2253-0216}} \\
  Center for AI and Data Science \\
  University of W{\"u}rzburg, Germany \\
  \texttt{ingo.scholtes@uni-wuerzburg.de} \\
  \And
  Christopher Bl{\"o}cker\thanks{Also at Center for AI and Data Science, University of W{\"u}rzburg, Germany}~~\mbox{\orcidlink{0000-0001-7881-2496}}\\
  Department of Computing Science \\
  Ume{\aa} University, Sweden \\
  \texttt{christopher.blocker@umu.se} \\
}

\begin{document}

\maketitle

\begin{abstract}
    Graph pooling is commonly applied in graph classification, yet its empirical gains over standard WL-1 expressive GNNs are often marginal or inconsistent.
    We study this gap by analysing the interaction between node features and graph topology and their effect on pooling objectives.
    Our analysis reveals that pooling operators require node features that are well-aligned with the graph's topology---a condition often overlooked and not guaranteed in empirical networks.
    We formalise fundamental requirements for node features to enable effective pooling, and introduce a quantitative measure of feature quality.
    Our empirical evaluation shows that, when these requirements are satisfied, pooling can be beneficial and improve performance on appropriate datasets.
\end{abstract}

\section{Introduction}
Graph Neural Networks (GNNs) achieve state-of-the-art performance across graph-learning tasks from various domains, such as molecular property prediction \cite{pmlr-v162-stark22a,10.5555/3692070.3694247}, and analysing scientific collaborations \cite{kipf2017semisupervised,bojchevski2018deep}.
The basic idea behind GNNs is simple yet powerful: They use the graph's topology for message passing (MP), and learn representations for nodes, edges, or entire graphs through iterative refinement \cite{bronstein2021geometricdeeplearninggrids}.
The messages passed between nodes are the nodes' embeddings, which initially correspond to the nodes' features.

Deep graph pooling reduces graphs stepwise, similar to image pooling operations in computer vision.
Graph pooling is applied in graph classification models that reduce entire graphs to vectorial representations~\cite{NEURIPS2018_e77dbaf6, pmlr-v119-bianchi20a} and in U-shaped GNN architectures~\cite{graph_unets}, however, it has not yet achieved the promising results seen in computer vision---benchmark results vary substantially between works.
Often, not using pooling operators yields results nearly as good as, or even better than, those obtained with pooling~\cite{pmlr-v119-bianchi20a,pmlr-sag_pool, castellana2025bnpoolbayesiannonparametricapproach}.
We attribute this behaviour to the interplay between node features and graph topology.
While GNNs pool nodes based on features or embeddings, they use the graph's topology to assess the quality of node clusters, but these two are generally not aligned.
However, graph pooling often silently assumes a correlation between node features and graph topology.
We show that this assumption is often violated, and demonstrate that features play an underestimated role in the proper functioning of pooling operators.

\clearpage
Specifically, our contributions are:
\begin{itemize}[itemsep=0pt,topsep=-2pt]
  \item We show that node features are the determining factor in the performance of pooling operators, a role often overlooked in deep graph pooling.
  \item We introduce a graph-colouring-based framework to assess the quality of node features and their suitability for graph pooling.
  \item We show that pooling is helpful in certain datasets. In those cases, incorporating positional encodings enhances pooling and improves performance.
  \item Different from previous works that blame poor pooling performance on the shortcomings of pooling operators, our results suggest that pooling suitability is an inherent dataset property.
\end{itemize}

\section{Related Work and Preliminaries}
\label{sec:related}

\textbf{Graph Pooling.}
Recent works have proposed a multitude of graph pooling methods that tackle graph pooling from different perspectives:
Assignment-based methods, such as DiffPool \cite{NEURIPS2018_e77dbaf6}, MinCut \cite{pmlr-v119-bianchi20a}, NOCD \cite{shchur2019overlapping}, DMoN \cite{JMLR:v24:20-998}, Neuromap \cite{NEURIPS2024_1f59562c}, JBPool \cite{Bianchi_2023_jb}, or BN-Pool \cite{castellana2025bnpoolbayesiannonparametricapproach}, create assignment matrices that group nodes into possibly overlapping clusters.
Some of these assignment-based methods build on community-detection approaches from network science \cite{doi:10.1073/pnas.0601602103,doi:10.1073/pnas.0706851105}.
Score-based methods, such as TopK-Pool \cite{graph_unets, nips19_understanding_att}, SAG-Pool \cite{pmlr-sag_pool, nips19_understanding_att}, EC-Pool \cite{ec_pool}, or k-MIS \cite{k_mis}, identify important nodes and choose them as representatives for relevant structures in the network.
All these methods can be cast into a generic framework that describes pooling as a combination of three operations: a learnable select ($\textsc{SEL}_\Theta$), reduce ($\textsc{RED}$), and connect ($\textsc{CON}$)~\cite{sel-red-con}.

\textbf{Weisfeiler-Leman Test.}
The Weisfeiler-Leman test is a graph isomorphism test related to the expressivity of GNNs \cite{xu2018how,10.1609/aaai.v33i01.33014602}.
It relies on the colour refinement procedure first defined in~\cite{Morgan1965MorganAlgorithm}.
We apply this colour-refinement method to argue about the distinguishability of nodes.
Unlike typical applications, we do not compare the stable colourings of two graphs; instead, we focus on the stable colours of nodes derived from the same initial colours within a single graph.

\textbf{Positional Encodings as Node Features.}
To equip GNNs with a sense of node locality and global position, various Positional Encodings (PEs) have been developed. 
Early works focused on leveraging the graph's geometry through the Laplacian spectrum~\citep{pe_laplacian}.
In contrast, structural approaches such as Random Walk PEs~\citep{pe_rw} and Node2Vec~\cite{10.1145/2939672.2939754} capture local topology by utilising transition probabilities and shallow node embeddings, respectively.
More recently, GPSE~\citep{pe_gpse} has introduced deep node embeddings by employing message-passing layers to predict structural descriptors.
Despite the increasing sophistication of these methods, they are rarely incorporated into the evaluation of graph pooling models.
In this work, we demonstrate the need to incorporate appropriate PEs to ensure a robust and equitable evaluation of deep graph pooling architectures.

\textbf{Spectral Clustering.}
In this work, we consider pooling from the perspective of spectral clustering, allowing us to remain within the same conceptual framework, clustering nodes with four approaches:
(1) Traditional spectral clustering algorithms are applied to point data, such as features.
(2) Applying such a traditional spectral clustering algorithm to the graph Laplacian leads to clusters defined by the graph's topology. \cite{fiedler73, 868688}
(3) Graph convolutional networks (GCN) perform spectral convolution \cite{kipf2017semisupervised}.
(4) MinCut applies a spectral clustering objective to train a GNN for community detection \cite{pmlr-v119-bianchi20a}.
To understand the failure modes of pooling, we study the assumption that node features are correlated with clusters \cite{pmlr-v119-bianchi20a}, however, as we will show, this is generally not the case.

\textbf{Preliminaries.}
We focus on community-based pooling methods, such as MinCut \cite{pmlr-v119-bianchi20a}, DiffPool \cite{NEURIPS2018_e77dbaf6}, DMoN \cite{JMLR:v24:20-998}, and MDL-Pool \cite{mdlpool}.
Let $G = \left(V, E\right)$ be a graph with nodes $V$ and links $E \subseteq V \times V$, adjacency matrix $\mathbf{A}$, and features $\mathbf{X}$.
Community-based pooling methods assign nodes to communities via a soft cluster-assignment matrix $\mathbf{S} \in \mathbb{R}^{\left|V\right| \times k}$, where the entry $\mathbf{S}_{ij}$ indicates how strongly node $v_i$ belongs to community $g_j$, and $k$ is the number of communities.
Often, but not always, these entries are constrained to sum to 1 per node via softmax.
Based on $\mathbf{S}$, the graph is then coarse-grained by aggregating node communities into ``super-nodes'' and summarising links accordingly.
However, these methods often produce disconnected communities, which is why we will refer to the identified node groups simply as \textit{groups}, rather than \textit{communities}, to avoid implying topological regularities.

Node groups can overlap, and nodes can belong to more than one group.
However, for simplicity, we assume non-overlapping groups where every node belongs to exactly one group, also called \textit{partitions}.
Then, $\mathbf{S}$ is binary and encodes group memberships, where $\mathbf{S}_{ij} = 1$ iff node $v_i$ belongs to group $g_j$, and $0$ otherwise.
Corresponding to $\mathbf{S}$, we define partition $P_{\mathbf{S}} = \left\{ g_j \right\}_{j=1}^{k}$ as induced by $\mathbf{S}$.
Conversely, for a given partition $P_\mathbf{S}$, we can construct the corresponding cluster assignment matrix $\mathbf{S}$; for simplicity, we will simply refer to partitions as $P$ and drop the index $\mathbf{S}$.

Community-based pooling methods aim to generate optimal group assignments and are trained via a combination of a supervised downstream task, typically based on graph-level labels, and an unsupervised pooling loss.
Pooling losses $\mathcal{L}_{topo}\left(G, \mathbf{S}\right)$ rely on the graphs' topology to determine the goodness of an assignment $\mathbf{S}$.
However, there is a gap between the working principles of GNNs and such topological loss functions:
While GNNs combine topology and node features to construct an assignment $\mathbf{S}$, the unsupervised pooling objective, $\mathcal{L}_{topo}\left(G, \mathbf{S}\right)$, usually ignores node features.

\section{Motivation -- When Does Pooling Work?} 
\label{sec:motivation}
\begin{figure*}[t!]
  \begin{overpic}[width=.15\linewidth]{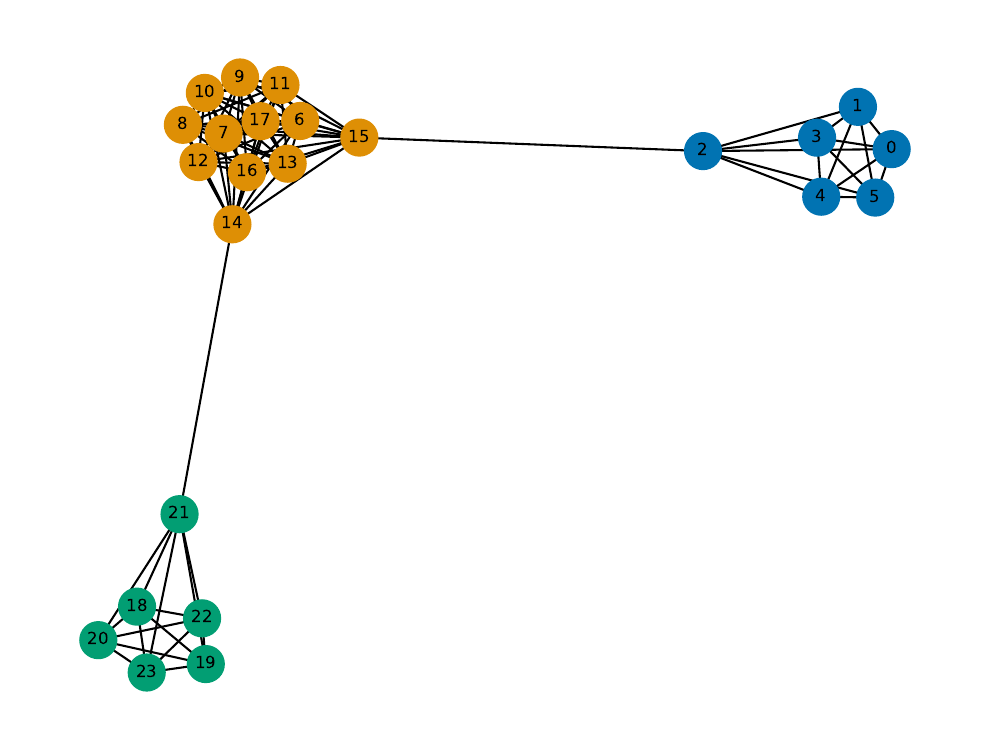}
    \put(-4,70){\textbf{(a)}}
  \end{overpic}
  \includegraphics[width=0.15\linewidth]{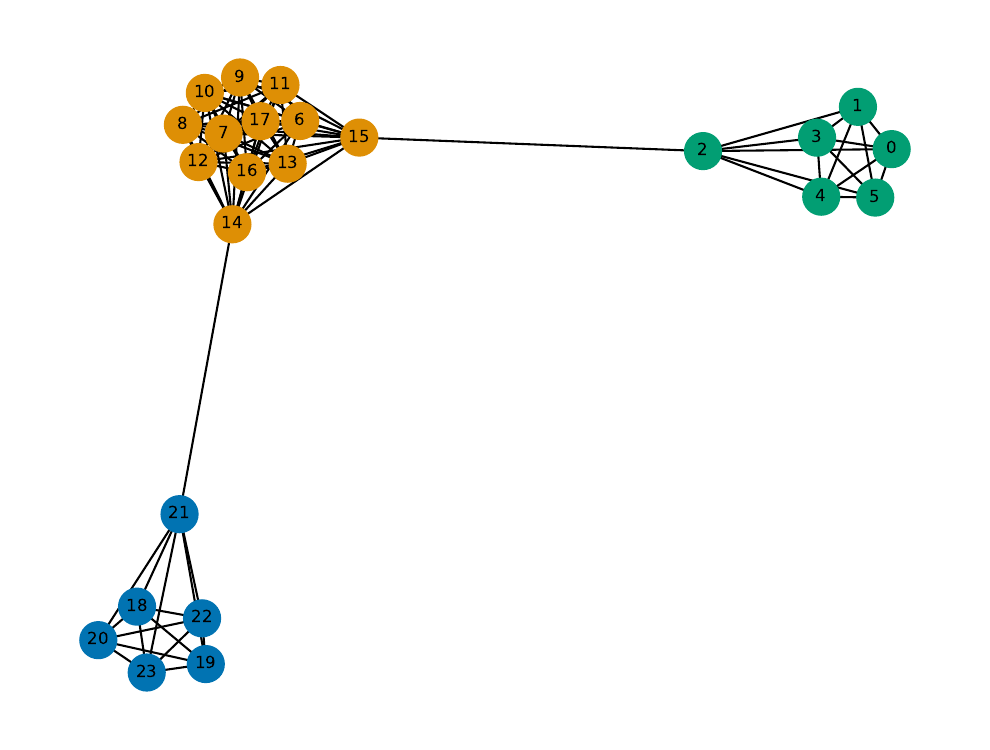}
  \quad
  \begin{overpic}[width=.15\linewidth]{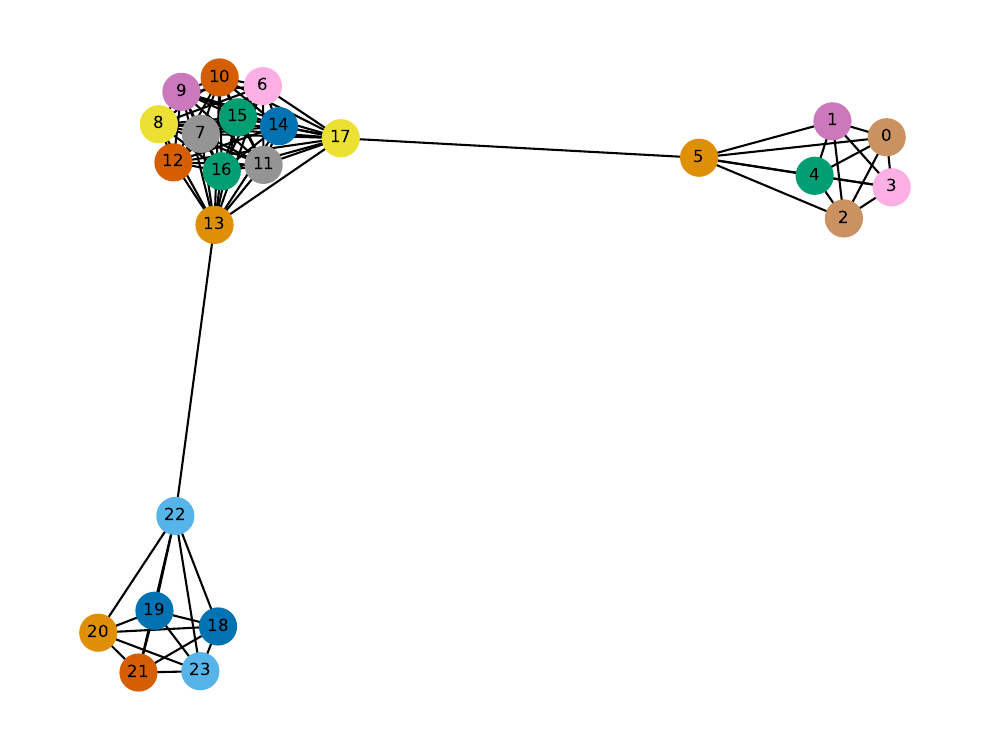}
    \put(-4,70){\textbf{(b)}}
  \end{overpic}
  \includegraphics[width=0.15\linewidth]{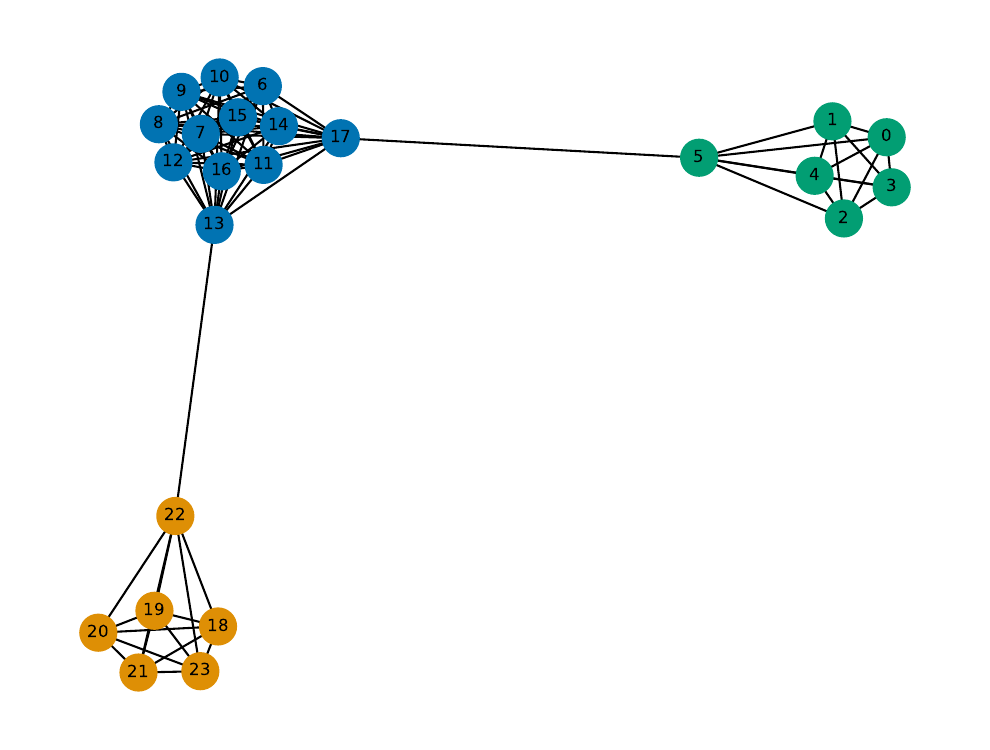}
  \quad
  \begin{overpic}[width=.15\linewidth]{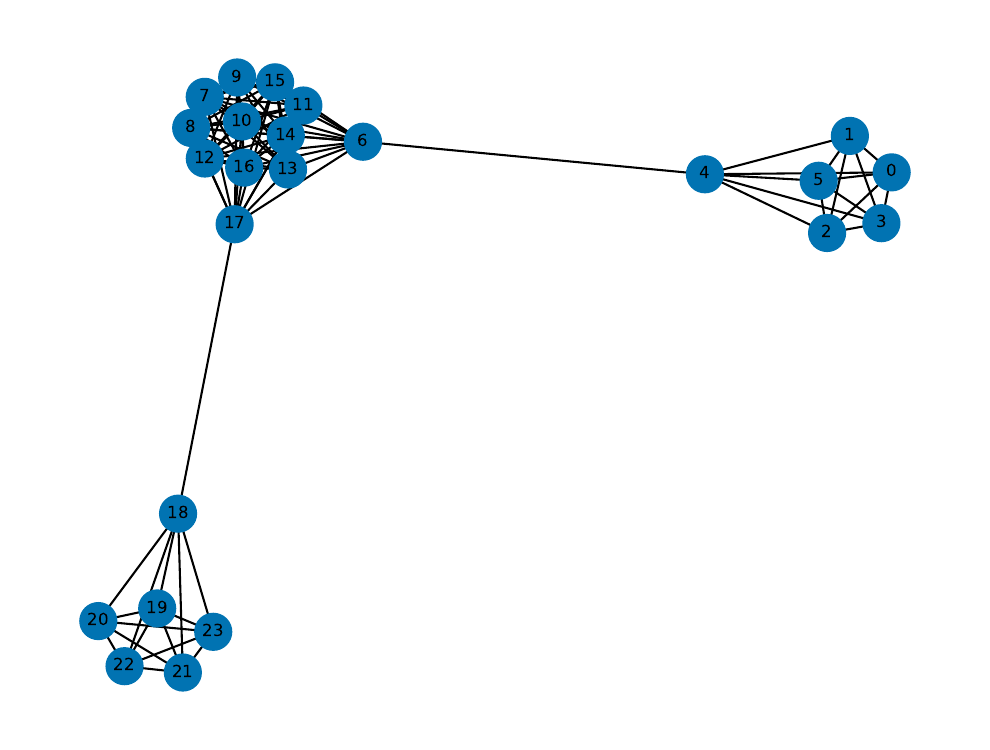}
    \put(-4,70){\textbf{(c)}}
  \end{overpic}
  \includegraphics[width=0.15\linewidth]{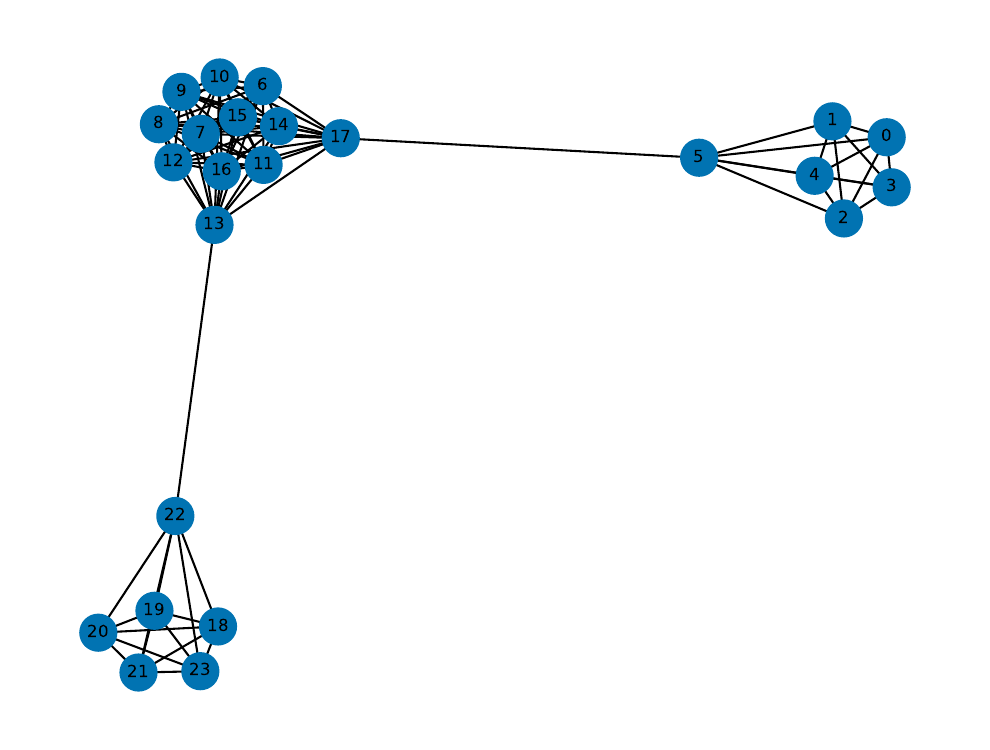}
    \caption{Three instances of a graph with different node features (left) and community assignments (right).
    We detect communities using MinCut-Pool, which employs a GCN and uses a spectral clustering objective.
    \textbf{(a)} Features correspond exactly to the optimal partition; MinCut-Pool works as expected.
    \textbf{(b)} When all nodes have unique features, MinCut-Pool finds the expected communities.}
    \textbf{(c)} When different communities contain nodes with the same features, MinCut-Pool fails.
    \label{fig:exp1}
\end{figure*}

Graph pooling assumes that we can coarse-grain graphs in a meaningful way to summarise their structural properties.
To find such summaries, GNNs combine graph topology with features to jointly learn node representations, assuming that combining them yields informative representations.
However, results from previous deep graph pooling works suggest that learning informative representations often fails.
Specifically, ``no-pool'' is a strong baseline that often outperforms more sophisticated pooling-based methods:
no-pool merely uses a GNN to learn node representations by optimising a global aggregation of them for a downstream task, but without assumptions about structural properties, such as clusters, or what characterises important nodes.

\begin{wrapfigure}[15]{r}{0.3\textwidth}
    \centering
    \vspace*{-1.5\baselineskip}
    \includegraphics[width=\linewidth]{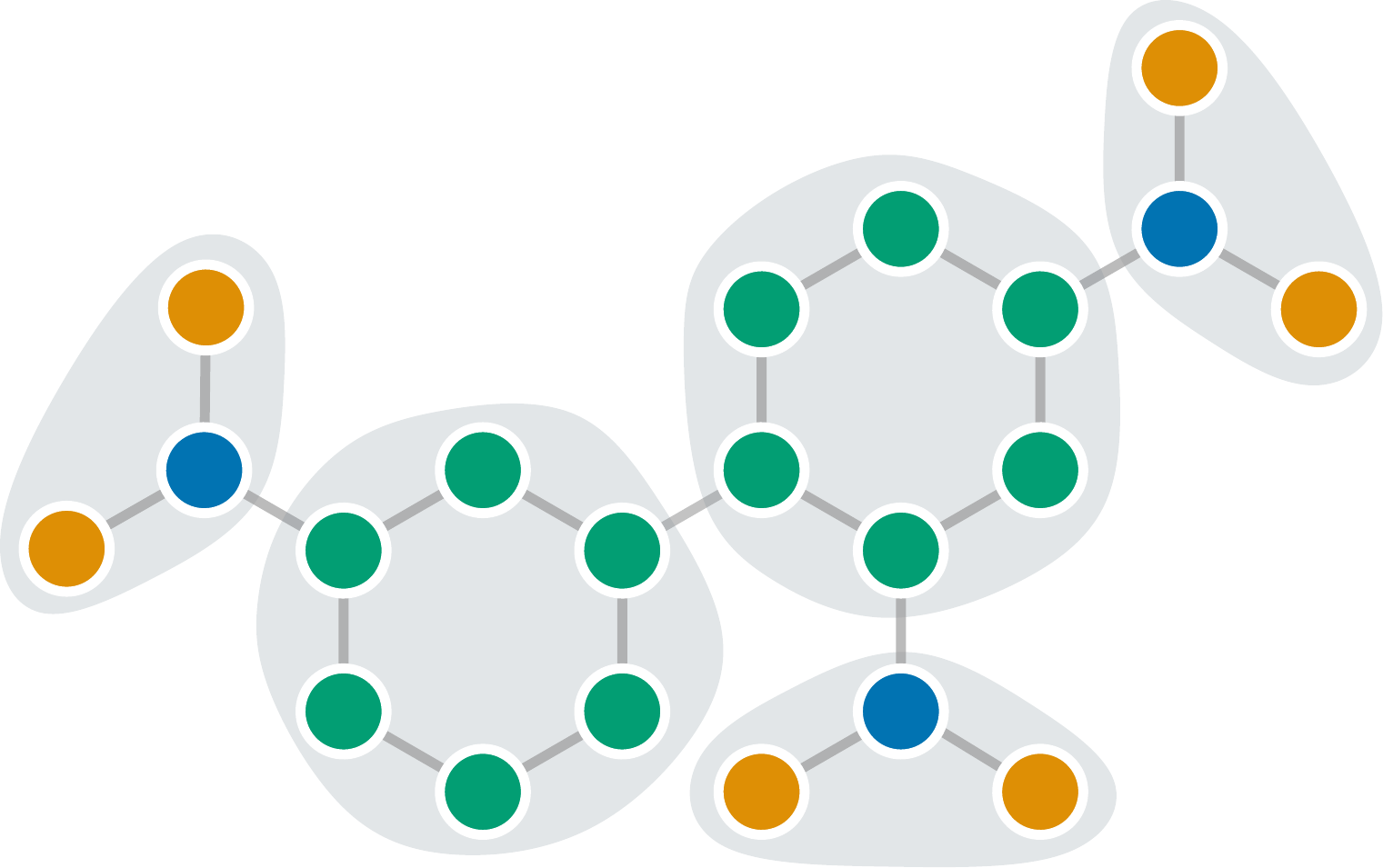}
    \caption{A graph from the Mutag dataset. The node features, shown in colours, do not align well with the topological communities, shown in grey and detected via spectral clustering of the graph Laplacian.
    }
    \label{fig:exp5}
\end{wrapfigure}
In this work, we study the interplay between topology and features, specifically, how it helps or hinders GNN pooling performance.
\Cref{fig:exp1} illustrates that, for good pooling performance, GNNs must be capable of assigning (1) nodes with dissimilar representations to the same cluster (\cref{fig:exp1}b), and (2) nodes with similar representations to different clusters (\cref{fig:exp1}c).
While the former is easy, the latter is challenging.
Why?
Consider clustering nodes with an MLP using the nodes' features, similar to traditional k-means clustering, which ignores topology.
Because the MLP implements a deterministic function, it must assign any two nodes with the same representation to the same cluster, regardless of their topological relationship.
Whether these clusters provide a meaningful summary of the graph's structure depends on the \textit{alignment between node features and topology}.
By combining features with topology, GNNs hold the promise to provide better clusters.
However, when nodes with the same features are situated in neighbourhoods that look alike, GNNs learn the same representation for those nodes and assign them to the same cluster, even if they are topologically far apart \cite{9311759}.
We analyse feature-topology alignments in empirical datasets, develop a notion of alignment between features and topology, and formalise the requirements for identifying optimal clusters.

\subsection{Feature-Topology Alignment in Empirical Networks}
How well do node features align with the topological communities in empirical networks?
We study this alignment and find that, unfortunately, \textbf{the alignment between node features and graph topology is generally poor in empirical datasets}.
As an example, consider an instance from the Mutag dataset~\citep{Morris+2020}, shown in \Cref{fig:exp5}.
For an empirical analysis, we use spectral clustering, denoted $\operatorname{SC}$, to cluster (1) the node features $\mathbf{X}$, and (2) the adjacency matrix $\mathbf{A}$, and measure their alignment with normalised mutual information (NMI) \citep{JMLR:v11:vinh10a} as $\operatorname{NMI}\left(\operatorname{SC}\left(\mathbf{A}\right), \operatorname{SC}\left(\mathbf{X}\right)\right)$.
We also consider clusters obtained with a GCN and MinCut loss, both of which view graph data from a spectral perspective, allowing us to stay within the same conceptual framework for our analyses.
\Cref{fig:nmi} shows the alignments between features, topology, and their combination using a GCN for six datasets, with and without Laplacian positional encodings (PEs).

\subsection{Requirements for generating optimal partitions.}
Here, we identify the key conditions required for a pooling operator $f$ to generate optimal assignments.
We consider $f\left(\mathbf{X}\right) = \sigma\left(\mathbf{HW}\right)$, where $\mathbf{W}$ are learnable weights, $\sigma$ is a non-linear activation function, here ReLU, and $\mathbf{H}$ are the node representations, typically learned with a GNN with $\mathbf{H} = \textsc{GNN}(\mathbf{A}, \mathbf{X})$.
However, for simplicity, we begin by using the raw node features as representations and set $\mathbf{H} = \mathbf{X}$.
\begin{wrapfigure}[19]{r}{0.5\textwidth}
    \centering
    \includegraphics[width=\linewidth]{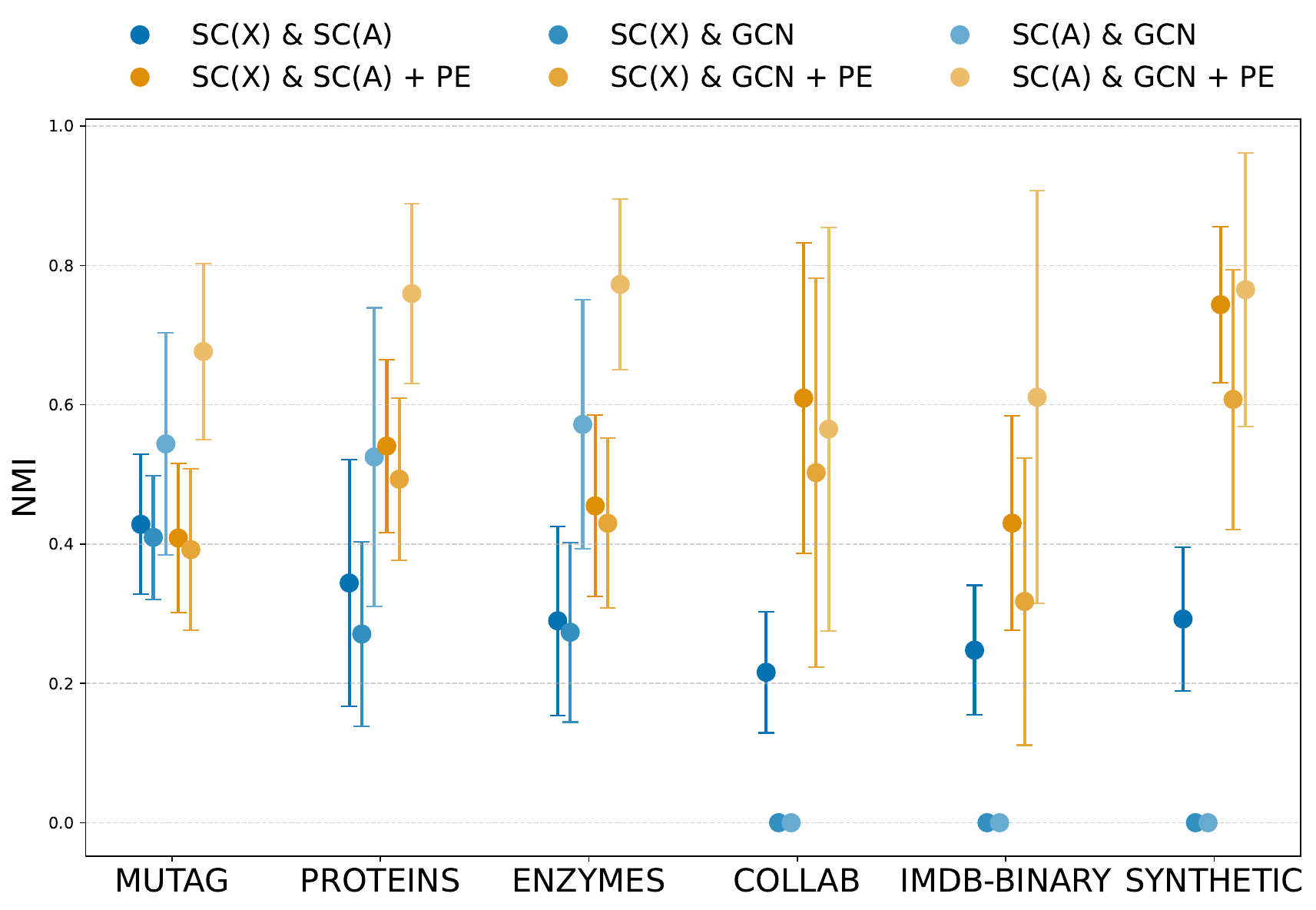}
    \caption{Pairwise normalised mutual information (NMI) between spectral clustering of features, $SC(\mathbf{X})$, topology, $SC(\mathbf{A})$, and communities via GCN and MinCut loss. We also tested including Laplacian positional encodings (PE).}
    \label{fig:nmi}
\end{wrapfigure}

\begin{definition}
    An optimal assignment operator $f_{opt}$ generates an optimal assignment $\mathbf{S}_{opt}$.
    An assignment $\mathbf{S}_{opt}$ is optimal for graph $G$ if there is no other $\mathbf{S}$ with $\mathcal{L}_{topo}(G, \mathbf{S}) < \mathcal{L}_{topo}\left(G, \mathbf{S}_{opt}\right)$.
\end{definition}

First, when the features correspond perfectly to the cluster labels for the optimal partition, there is a trivial mapping $f$ that produces $\mathbf{S}_{opt}$.
\Cref{fig:exp1}a shows an example where features $\mathbf{X}$ correspond to the desired partition, and clustering with a GCN and MinCut loss returns $\mathbf{S}_{opt}$.
\begin{proposition}
\label{theorem:align}
    Given a graph G with node features $\mathbf{X}$ and an optimal assignment $\mathbf{S}_\text{opt}$, there is a $\mathbf{W}$ such that $f(\mathbf{X}) = \mathbf{S}_\text{opt}$ if (1) any two nodes that have the same features belong to the same group, and (2) all nodes in the same group have the same features, that is, $\forall v_i, v_j \in V\colon x_i = x_j \Leftrightarrow g_i = g_j$.
\end{proposition}
\begin{proof}
By assumption, there is a bijection between node features and groups.
Therefore, we can represent the node features in a one-hot indicator matrix $\mathbf{X}'$ that is, up to permutation, equal to the optimal assignment $\mathbf{S}_\text{opt}$.
Setting $\mathbf{W} = \textbf{I}$ produces the desired groups $\mathbf{S}_\text{opt}$ from $\mathbf{X}'$.
\end{proof}

\begin{wrapfigure}[16]{r}{0.5\textwidth}
    \centering
    \vspace*{-1\baselineskip}
    \begin{overpic}[width=\linewidth]{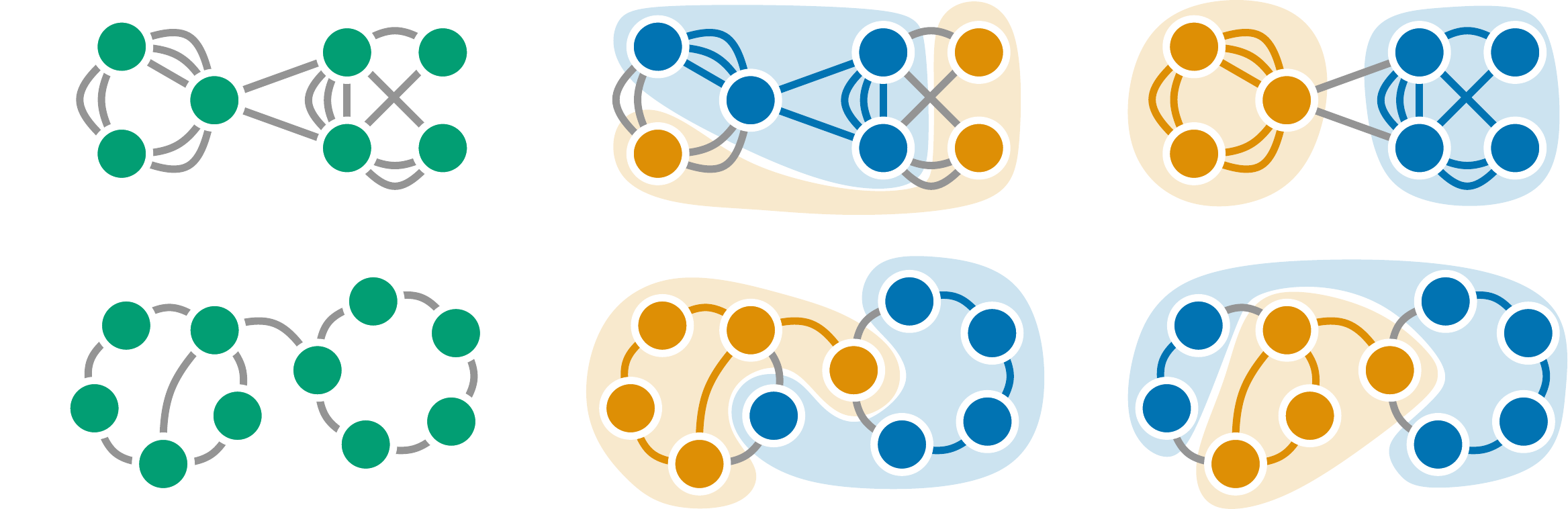}
        \put(0,30){\textbf{(a)}}
        \put(32.5,30){\textbf{(b)}}
        \put(67,30){\textbf{(c)}}
    \end{overpic}
    \caption{Clustering networks with a GCN and MinCut loss. The detected clusters depend on the initial node features and the number of GNN layers.
    \textbf{(a)} Raw node features and topological communities. We use multiedges to keep the examples simple.
    \textbf{(b)} Two layers: the GCN cannot detect the communities because the GCN produces identical representations for nodes in different communities.
    \textbf{(c)} Five layers: the GCN can detect the communities in the top graph, but not in the bottom graph.
    }
    \label{fig:clustering-mincut}
\end{wrapfigure}
Second, relaxing Proposition~\ref{theorem:align}, $f$ can generate the optimal assignment $\mathbf{S}_\text{opt}$ even when nodes in the same group have different features, as long as no two nodes in different groups have the same features.
This is the case for a GCN trained with MinCut loss as shown in \Cref{fig:exp1}b.
\begin{proposition}
\label{theorem:injective}
    Given features $\mathbf{X}$, graph $G$, and optimal assignment $\mathbf{S}_\text{opt}$, there is a $\mathbf{W}$ such that 
    $f\left(X\right) = \mathbf{S}_\text{opt}$
    if nodes with the same features are in the same group, $\forall v_i, v_j \in V\!\colon x_i = x_j \Rightarrow g_i = g_j$.
\end{proposition}

\begin{proof}
    Consider two nodes $v_i, v_j$ in the same group $g$ but with different features, $x_i \neq x_j$.
    As before, we represent the node features in a one-hot indicator matrix $\mathbf{X}'$, and set $\mathbf{W} = \mathbf{I}$ to assign nodes with identical features to the same group.
    We further select a column $\mathbf{W}_{\cdot,i}$ corresponding to any one node $v_i \in g$ for each group $g$ and change the other corresponding columns $v_j \in g, j \neq i$ of the given group to $\mathbf{W}_{\cdot,j} = \mathbf{W}_{\cdot,i}$. Consequently, $f$ maps all nodes in $g$ to the same group vector defined by $v_i$.
\end{proof}

Third, when nodes in different groups share the same features, the deterministic function $f$ assigns them to the same group; $f$ cannot generate $\mathbf{S}_\text{opt}$.
\Cref{fig:exp1}b illustrates this for the MinCut objective, where optimising a GCN with the MinCut loss does not lead to the optimal assignment.
\begin{proposition}
\label{theorem:deterministic}
    Given features $\mathbf{X}$, graph $G$, and an optimal assignment $\mathbf{S}_\text{opt}$, if there are nodes $v_i \in g_i, v_j \in g_j$ that belong to different groups $g_i \neq g_j$, but have the same features, $x_i = x_j$, then there is no $\mathbf{W}$ such that $f\left(\mathbf{X}\right) = \mathbf{S}_\text{opt}$.
\end{proposition}
\begin{proof}
    Consider two nodes $v_i, v_j$ with the same features, $x_i = x_j$, but in different groups, $g_i \neq g_j$.
    Because $f$ categorises nodes based on their representations alone, we have $f\left(x_i\right) = f\left(x_j\right)$, regardless of $\mathbf{W}$, and it is impossible to assign $v_i$ and $v_j$ to different groups as required by $\mathbf{S}_\text{opt}$.
\end{proof}

\begin{proposition}
    It follows from Propositions~\ref{theorem:deterministic} and \ref{theorem:injective} that $f$ can construct the optimal assignment if node representations are unique, that is, when no two nodes have the same representation.
\end{proposition}

\Cref{theorem:align,theorem:deterministic,theorem:injective} provide simple criteria that explain when $f$ can produce the optimal assignment $\mathbf{S}_\text{opt}$, or when this is impossible.
So far, $f$ operated on the raw node features $\mathbf{X}$; in practice, pooling uses GNNs to jointly learn node representations $\mathbf{H}$ from node features $\mathbf{X}$ and topology $\mathbf{A}$ via message passing.
Regardless, the same limitations apply when assignments are based on $\mathbf{H}$.
GNNs may learn the same representation $h$ for nodes that belong to different clusters.
Contrary to intuition, we show in \Cref{fig:clustering-mincut} that increasing the number of message passing iterations---so the GNN has sufficient ``reach'' to observe the entire graph---does not solve the problem.
In the following, \textbf{we consider how node features can be aligned with the topology so they serve as an informative proxy}.

\section{Methodology -- Colouring-based Feature Evaluation}
\label{sec:method}
We motivated why community-based pooling methods rely on the alignment between node features and graph topology.
However, as we have seen, alignment cannot be assumed for empirical data.
In this section, we develop a measure to assess the quality of a dataset's features for graph pooling.
We interpret node features as colours, and relate the quality of a colouring to the minimum possible number of ``mis-assignments'' for mappings from colours to communities.
We also characterise the transferability of colours from seen to unseen graphs, which is crucial for graph pooling in the inductive setting.
Our notion of transferability leads us to requirements for encodings that can improve pooling performance. 

Using colours as proxies for node features, we relate the quality of a colouring to whether it enables a pooling operator $\operatorname{SEL}_\Theta$ to make correct predictions, that is, predictions $\hat{y}$ that minimise the discrepancy with ground truth $y$.
A node colouring $\zeta \colon V \mapsto C$ is a mapping that assigns a colour $\zeta \left(v_i\right) \in C$ to each node $v_i \in V$.
We represent node colours as one-hot encodings, set $f = \textsc{SEL}_\Theta$, and recap Propositions~\ref{theorem:align} to \ref{theorem:deterministic} to characterise the relationship between a colouring $\zeta$ and partition~$P$:
\begin{description}
    \item[By Proposition~\ref{theorem:align}:] All nodes $V_c \subseteq V$ with the same colour $c \in C$ \underline{can} be mapped to the same group $g \in P$ because $\textsc{SEL}_\Theta$
    learns a mapping.
    \item[By Proposition~\ref{theorem:injective}:] Nodes with different colours $V_{c_1}, V_{c_2} \subseteq V, c_1 \neq c_2$ \underline{can} be mapped to the same group $g \in P$ if $\textsc{SEL}_\Theta$
    is not injective, which holds for most neural network implementations of $\textsc{SEL}_\Theta$.
    \item[By Proposition~\ref{theorem:deterministic}:] All nodes $V_c \subseteq V$ with the same colour $c \in C$ \underline{must} be mapped to the same group $g \in P$ because $\textsc{SEL}_\Theta$
    is deterministic. 
\end{description}

\subsection{Colouring Validity}
\begin{definition}
    Given a graph $G = \left(V,E\right)$ and a colouring $\zeta \colon V \mapsto C$, a colour $c \in C$ is \textbf{valid} if all nodes with that colour belong to the same group.
    Furthermore, a colouring $\zeta$ is valid if all colours $c \in C$ are valid, that is, $\forall c \in C \colon \exists g \in P\colon g \supseteq \left\{ v \in V \,|\, \zeta \left(v\right) = c \right\}$.
\end{definition}
\Cref{fig:valid-match} shows examples of valid and invalid colourings.
We measure validity using the indicator 
\begin{equation}
    \mathbf{1}_\zeta \left(c, g\right) = \left| \left\{\zeta \left(v\right) \,|\, v \in g \right\} \cap \left\{c\right\} \right|
\end{equation}
which, given colouring $\zeta$, is $1$ if group $g$ contains a node with colour $c$, and $0$ otherwise.
The number of groups in partition $P$ that contain nodes with colour $c$ is
\begin{equation}
    \gamma \left(c, P \,|\, \zeta\right) = \sum_{g \in P} \mathbf{1}_\zeta \left(c, g\right),
    \label{eqn:validity}
\end{equation}
and the fraction of valid colours, or the relative validity of colouring $\zeta$, is
\begin{align}
    \Gamma \left(\zeta \,|\, P\right) = \frac{ \left| \left\{ c \in \mathcal{C} \,|\, \gamma \left(c, P \,|\, \zeta\right) = 1 \right\} \right| }{ \left| C \right| }
\end{align}
We designed $\Gamma\left(\zeta \,|\, P\right)$ as a measure of how well colouring $\zeta$ reflects partition $P$, with the property:
\begin{theorem}
    Given a graph $G$ with a partition $P$ and a colouring $\zeta$, $\Gamma\left( \zeta \,|\, P\right)$ is the probability that a randomly chosen colour $c \in \mathcal{C}$ is valid, that is, that $c$ appears in exactly one group.
\end{theorem}
Clearly, when all nodes have unique colours, we have $\Gamma\left( \zeta \,|\, P\right) = 1$ because $\gamma\left( c, P\right) = 1 ~ \forall c \in C$.
To analyse how colouring validity depends on the number of colours, we use molecules from the \nobreak{Mutag} dataset as an example.
We assign different numbers of node colours uniformly at random, use $r \in \left\{0, 1, 2, 3, \infty\right\}$ Weisfeiler-Leman colour-refinement (CR) iterations, and compare to optimal spectral partitions.
In \Cref{fig:qrqr}a, we show the average $\Gamma\left( \zeta \,|\, P\right)$ as a function of the relative number of colours, $\nicefrac{k}{\left|V\right|}$.
We find that colouring validity increases with the number of available colours, which experimentally supports our argument that the most colourful colouring enables the optimal partition.

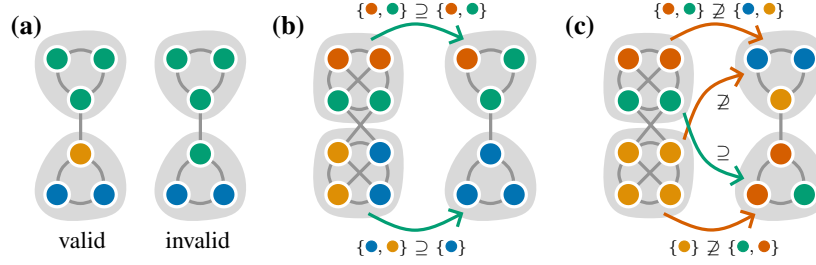
\begin{figure}[t]
    \vspace*{-.5\baselineskip}
    \centering
    \begin{overpic}[width=.75\linewidth]{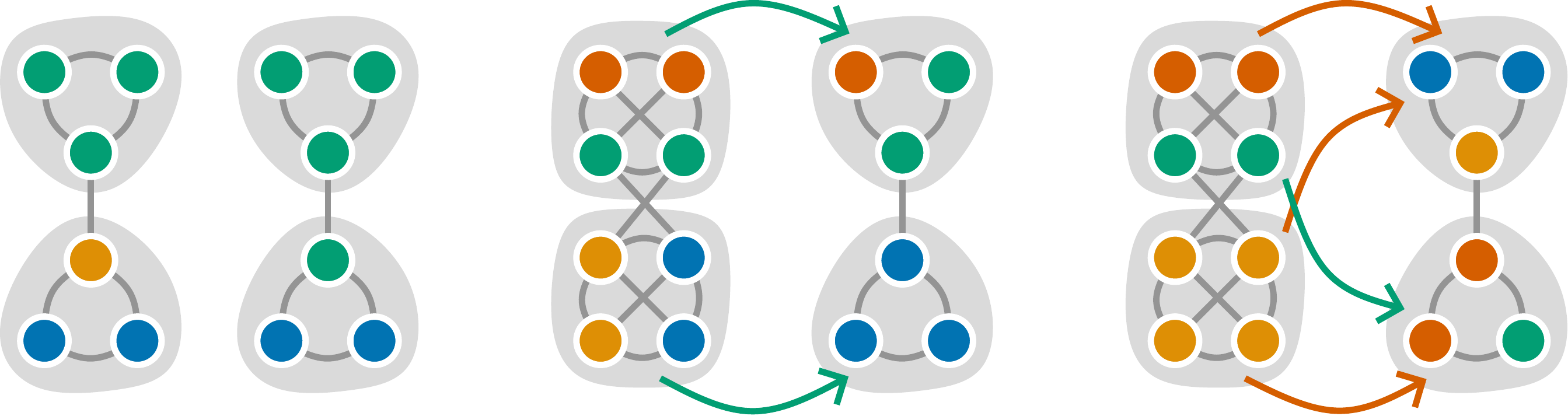}
        \put(-3,25){\textbf{(a)}}
        \put(3,-2){\small{valid}}
        \put(16.5,-2){\small{invalid}}
        
        \put(30,25){\textbf{(b)}}
        \put(40.5,27.5){\tiny$\left\{
            \tikz \filldraw[draw=palered,fill=palered] (0,0) circle (2pt);,
            \tikz \filldraw[draw=fadedgreen,fill=fadedgreen] (0,0) circle (2pt);
        \right\} \supseteq \left\{
            \tikz \filldraw[draw=palered,fill=palered] (0,0) circle (2pt);,
            \tikz \filldraw[draw=fadedgreen,fill=fadedgreen] (0,0) circle (2pt);
        \right\}$}
        \put(40.5,-2.5){\tiny$\left\{
            \tikz \filldraw[draw=windowsblue,fill=windowsblue] (0,0) circle (2pt);,
            \tikz \filldraw[draw=dustyorange,fill=dustyorange] (0,0) circle (2pt);
        \right\} \supseteq \left\{
            \tikz \filldraw[draw=windowsblue,fill=windowsblue] (0,0) circle (2pt);
        \right\}$}

        \put(67,25){\textbf{(c)}}
        \put(78,27.5){\tiny$\left\{
            \tikz \filldraw[draw=palered,fill=palered] (0,0) circle (2pt);,
            \tikz \filldraw[draw=fadedgreen,fill=fadedgreen] (0,0) circle (2pt);
        \right\} \not\supseteq \left\{
            \tikz \filldraw[draw=windowsblue,fill=windowsblue] (0,0) circle (2pt);,
            \tikz \filldraw[draw=dustyorange,fill=dustyorange] (0,0) circle (2pt);
        \right\}$}
        \put(80,-2.5){\tiny$\left\{
            \tikz \filldraw[draw=dustyorange,fill=dustyorange] (0,0) circle (2pt);
        \right\} \not\supseteq \left\{
            \tikz \filldraw[draw=fadedgreen,fill=fadedgreen] (0,0) circle (2pt);,
            \tikz \filldraw[draw=palered,fill=palered] (0,0) circle (2pt);
        \right\}$}
        \put(86,9.5){\tiny$\supseteq$}
        \put(86,16){\tiny$\not\supseteq$}
    \end{overpic}
    \vspace*{1\baselineskip}
    \caption{Colouring validity and transferability.
    \textbf{(a)} Colouring examples: a colouring is valid if every colour appears in at most one group.
    \textbf{(b)} Matching colourings: groups can be transferred from the source/seen partition (left) to the target/unseen partition (right) because, for every group in the target partition, there is a group in the source partition whose colours are a superset of the target group's colours.
    \textbf{(c)} Non-matching colourings: groups cannot be transferred because there is a group in the target partition (right) whose colours are not a subset of any group in the source partition (left).
    }
    \label{fig:valid-match}
\end{figure}

\subsection{Colouring Transferability}
In the inductive setting of graph pooling, $\textsc{SEL}_\Theta$ is trained and evaluated on disjoint sets of graphs, meaning $\textsc{SEL}_\Theta$ must be transferable to unseen graphs.
We introduce the notion of \textit{matching} colourings for a seen $G_s$ and unseen graph $G_u$, measuring if a similar group $\hat{g} \in G_s$ was seen for $g \in G_u$.
\begin{definition}
    Let $G_u$ and $G_s$ be graphs with partitions $P_u, P_s$, and colourings $\zeta_u, \zeta_s$, respectively.
    A group $g_u \in P_u$ \textbf{matches} a group $g_s \in P_s$ if its colours are a subset of $g_s$' colours.
    Furthermore, colouring $\zeta_u$ matches $\zeta_s$ if there is a map $\mu \colon P_u \mapsto P_s$ such that all groups match the mapped group.
\end{definition}

\Cref{fig:valid-match} shows examples of matching and non-matching colourings.
To define a measure of how well colourings match, we first introduce the indicator 
\begin{align}
\mathbf{1}_{\zeta_s, \zeta_u}\left(g_s, g_u\right) = \begin{cases}
    1 \quad \text{ if } \{\zeta_u(v)\,|\, v \in g_u\} \subseteq \{\zeta_s(v)\,|\, v \in g_s\} \\
    0 \quad \text{ otherwise }
\end{cases}
\end{align}
that shows whether an unseen group $g_u$ matches a seen group $g_s$ given the colourings $\zeta_u$ and $\zeta_s$.
We count how many groups in $P_s$ match the unseen group $g_u$,
\begin{equation}
    \lambda(P_s, g_u \,|\, \zeta_s, \zeta_u) = \sum_{g_s \in P_s} \mathbf{1}_{\zeta_s, \zeta_u}(g_s, g_u),
\end{equation}
and define $\Lambda\left( \zeta_s, \zeta_u \,|\, P_s, P_u \right)$, the indicator of matching colourings between $P_s$ and $P_u$, as:
\begin{align}
    \Lambda\left( \zeta_s, \zeta_u \,|\, P_s, P_u\right) = \begin{cases}
        1 \text{ if } \lambda\left( P_s, g_u \,|\, \zeta_s, \zeta_u \right) \geq 1 \,\forall\, g_u \in P_u
        \\
        0 \text{ otherwise}
    \end{cases}
    \label{eqn:transferability}
\end{align}
The measure $\Lambda\left( \zeta_s, \zeta_u \,|\, P_s, P_u \right)$ is designed to have the following property:
\begin{theorem}
    Given a set of unseen graphs $G_u \in \mathcal{G}_u$ and a $G_s$, $\sum_{G_u \in \mathcal{G}_u} \Lambda\left(\zeta_s, \zeta_u \,|\, P_s, P_u\right) / \left| \mathcal{G}_u \right|$ is the probability that the colouring $\zeta_u$ of a randomly chosen graph $G_u \in \mathcal{G}_u$ matches the seen $\zeta_s$.
\end{theorem}
Colourings $\zeta_u, \zeta_s$ that assign the same colour to every node are trivially perfectly transferable for any $P_s, P_u$ with $\Lambda\left(\zeta_s, \zeta_u \,|\, P_s, P_u\right) = 1$.
In \Cref{fig:qrqr}a-b, we show that transferability generally suffers with more colours, whereas colouring validity increases with more colours.
Our definition of transferability in \Cref{eqn:transferability} is binary---a colouring is either transferable or it is not---which may be a limitation in practice; we consider continuous variants of transferability in \Cref{sec:variants}.

\subsection{Combined Colouring Quality}
We define the combined colouring quality $Q$ as the limiting factor of $\Gamma\left( \zeta \,|\, P\right)$ and $\Lambda\left( \zeta_s, \zeta_u \,|\, P_s, P_u\right)$,
\begin{align}
    Q \left( \zeta_s, \zeta_u \,|\, P_s, P_u\right) = \min \left( \Gamma\left( \zeta_s \cup \zeta_u \,|\, P_s \cup P_u \right), \Lambda\left( \zeta_s, \zeta_u \,|\, P_s, P_u \right) \right).
    \label{eqn:quality}
\end{align}
$Q$ can be computed in $\mathcal{O}(|V_u| |P_s| + |V_s|\log |C_s|)$; see \Cref{sec:variants} for details.

\begin{figure*}[t]
  \centering
  \begin{overpic}[width=.23\linewidth]{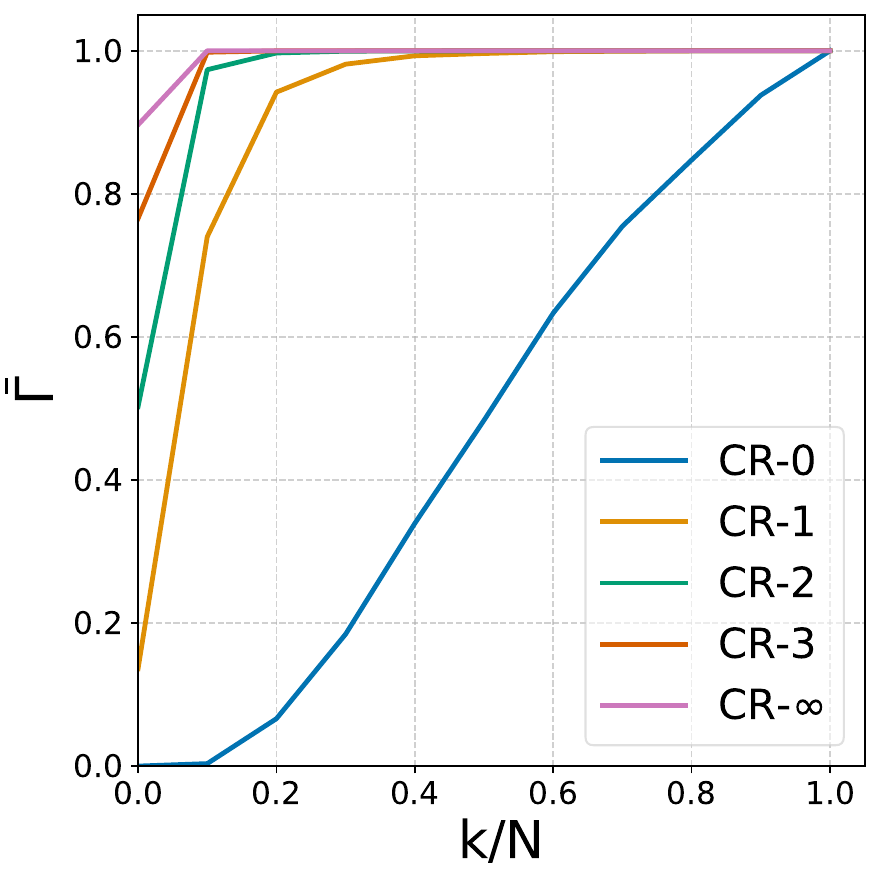}
    \put(-5,95){\textbf{(a)}}
  \end{overpic}
  \hfill
  \begin{overpic}[width=.23\linewidth]{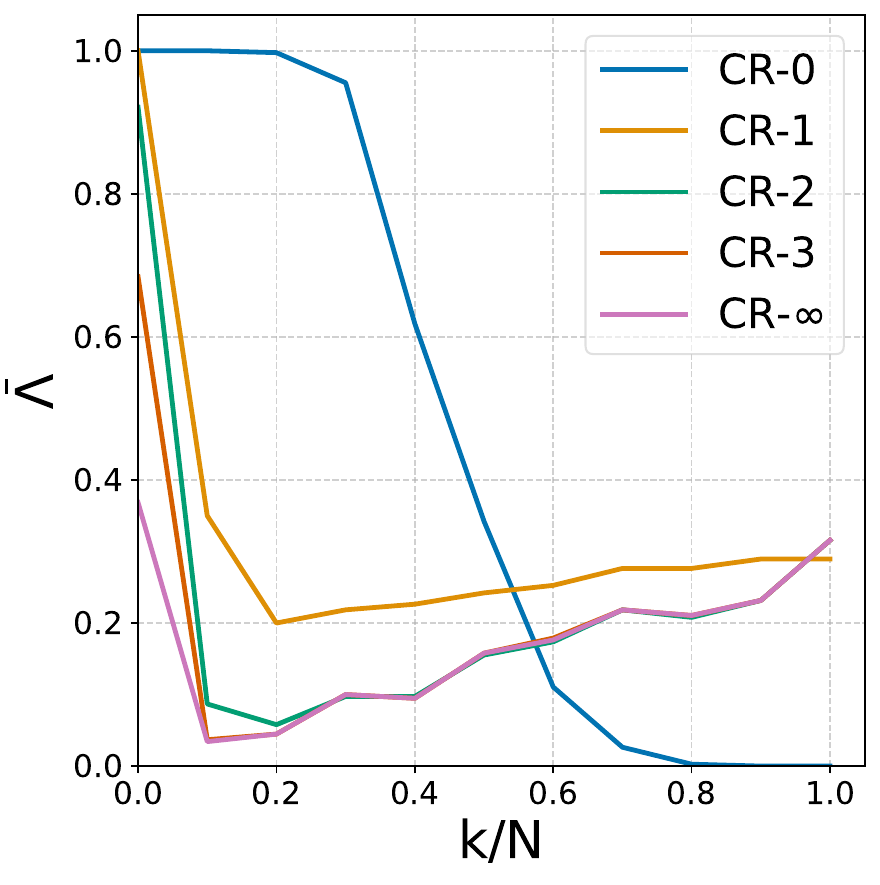}
    \put(-5,95){\textbf{(b)}}
  \end{overpic}
  \hfill
  \begin{overpic}[width=.23\linewidth]{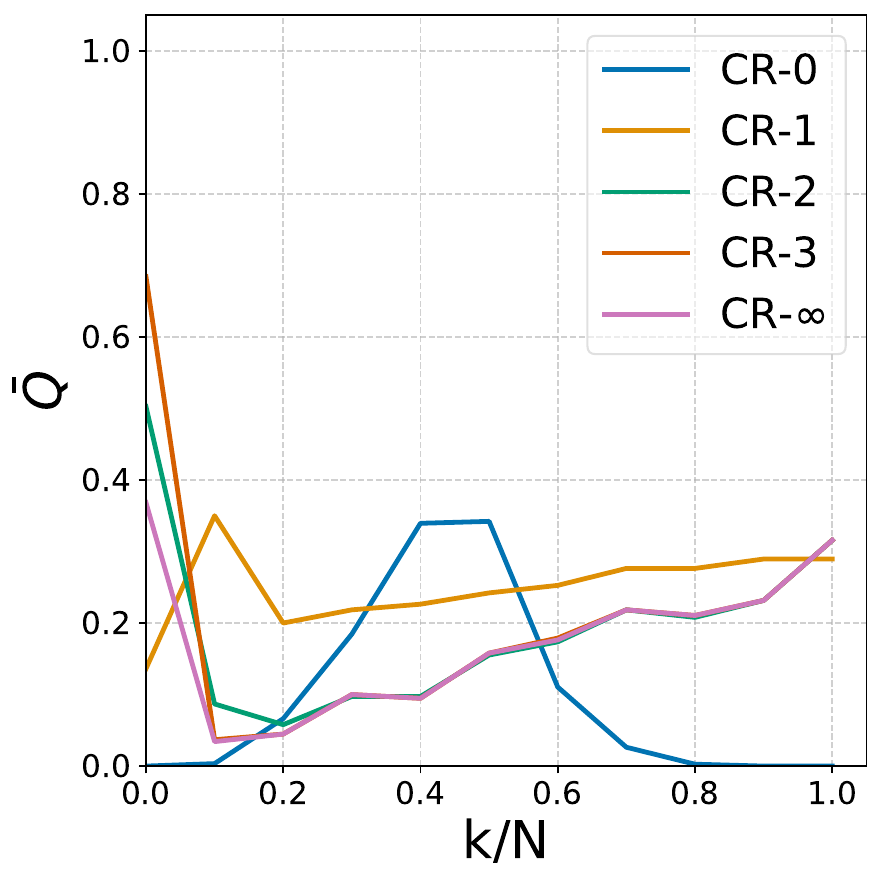}
    \put(-4,95){\textbf{(c)}}
  \end{overpic}
  \hfill
  \begin{overpic}[width=.23\linewidth]{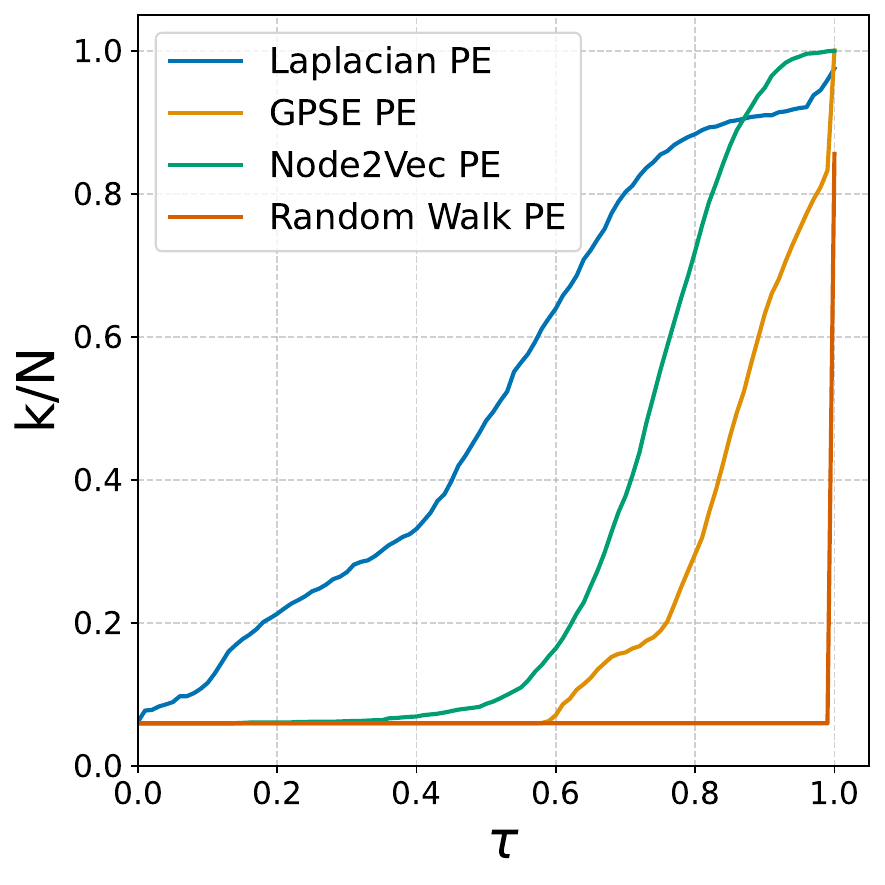}
    \put(-6,95){\textbf{(d)}}
  \end{overpic}
  \vspace*{-.5\baselineskip}
  \caption{\textsc{Mutag} dataset: \textbf{(a)} average colouring validity, \textbf{(b)} transferability, and \textbf{(c)} combined quality as a function of the relative number of colours $k/N$ with respect to spectral partitions of the graphs' adjacency matrices. We assign $k$ colours at random to $N$ nodes and apply $\{1,2,3,4,\infty\}$ colour-refinement steps. \textbf{(d)} Relative number of colours for different PEs as a function of threshold $\tau$.
  }
  \vspace*{-1.\baselineskip}
\label{fig:qrqr}
\end{figure*}

\subsection{Colour Refinement and Feature Resolution in GNNs}
\begin{wrapfigure}[16]{r}{0.40\textwidth}
    \vspace*{-.25\baselineskip}
    \centering
    \begin{overpic}[width=.4\linewidth]{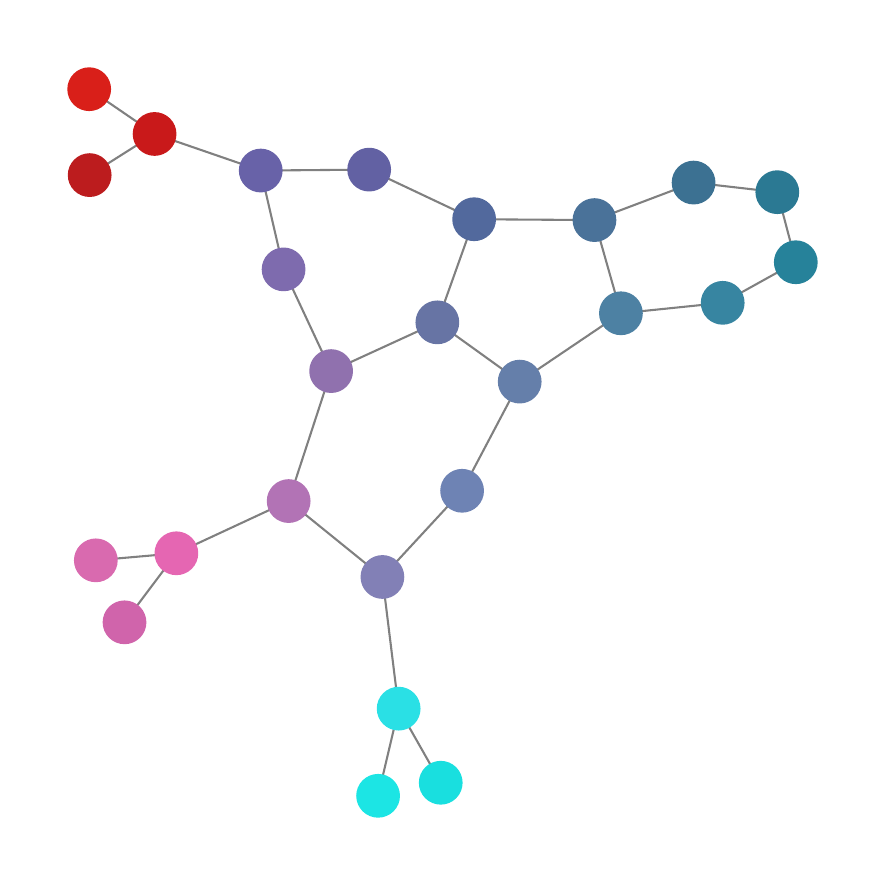}
        \put(-10,100){\textbf{(a)} Laplacian}
    \end{overpic}
    \quad
    \begin{overpic}[width=.4\linewidth]{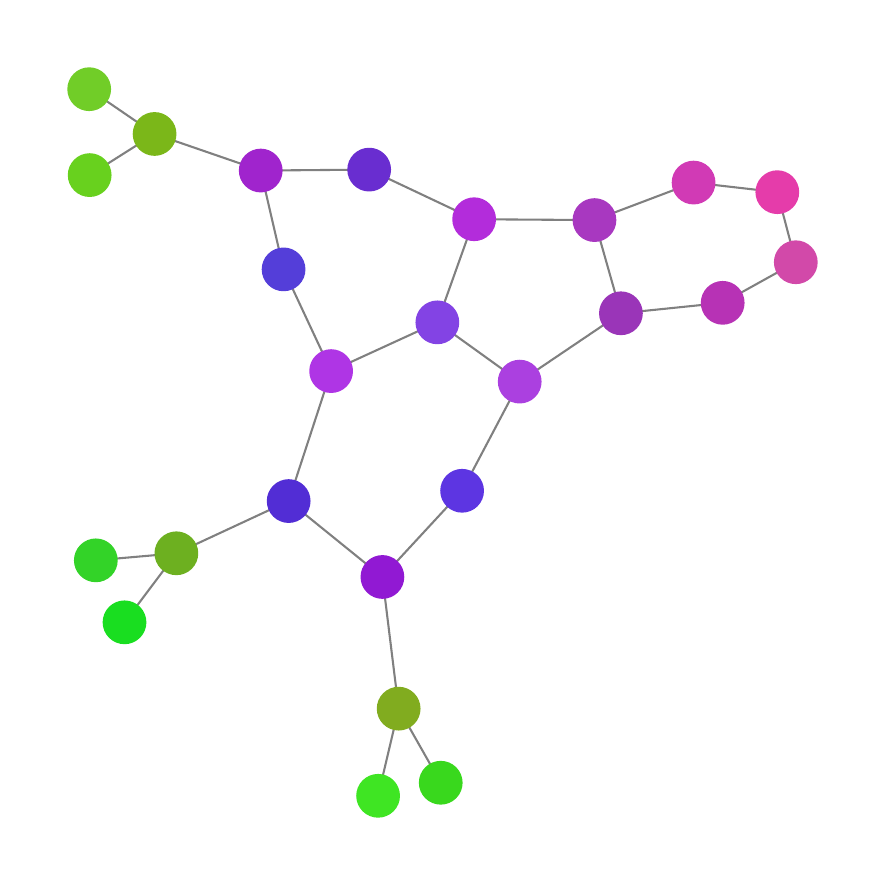}
        \put(-10,100){\textbf{(b)} GPSE}
    \end{overpic}\\
    \vspace*{.75\baselineskip}
    \begin{overpic}[width=.4\linewidth]{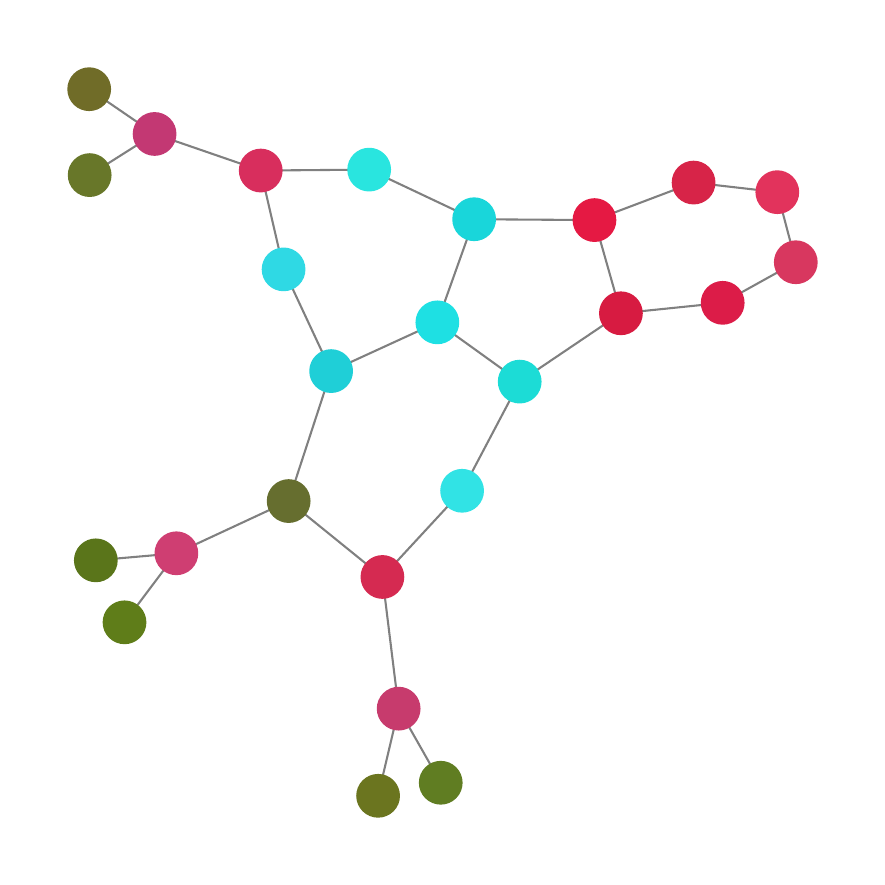}
        \put(-10,100){\textbf{(c)} Random Walk}
    \end{overpic}
    \quad
    \begin{overpic}[width=.4\linewidth]{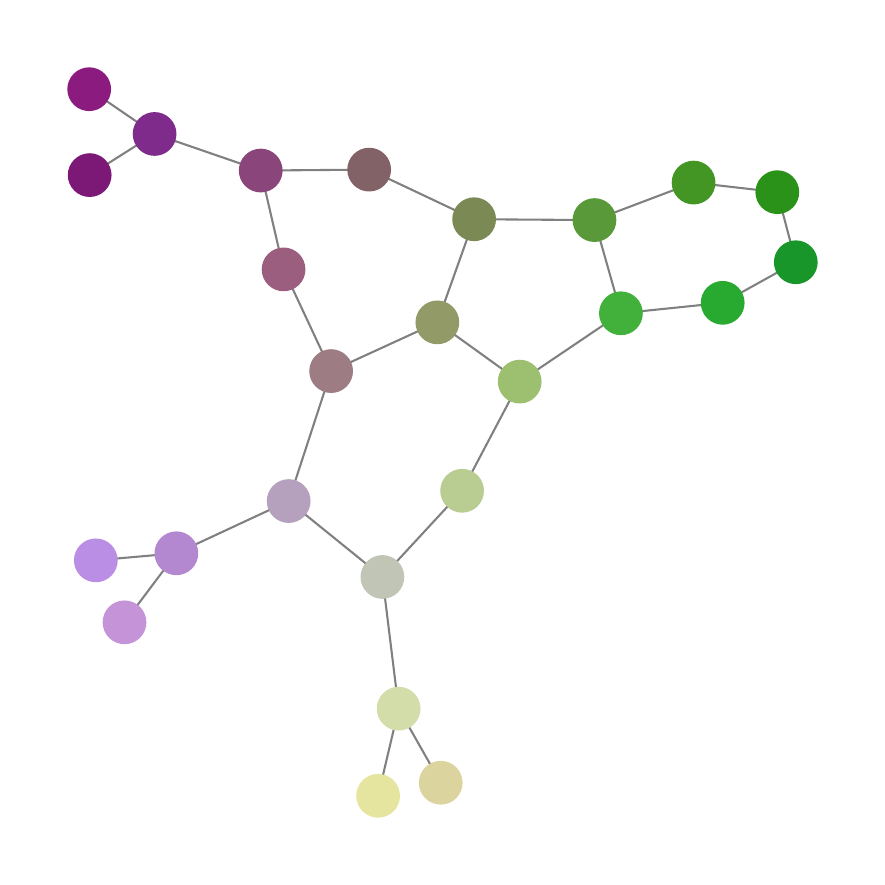}
        \put(-10,100){\textbf{(d)} node2vec}
    \end{overpic}
    \vspace*{-.5\baselineskip}
    \caption{Positional encodings on an example graph from the Mutag dataset.}
    \label{fig:pe-examples}
    \vspace{-1em}
\end{wrapfigure}
Moving beyond interpreting nodes' raw features $\mathbf{X}$ as colours, and using them to assign nodes to groups, we ask:
How do GNNs affect the colours, their validity, and transferability?
GNNs combine node features $\mathbf{X}$ with the graph's topology $\mathbf{A}$ to obtain node representations $\mathbf{H}$, which we interpret as topology-refined features.
Because GNNs used for pooling are at most as expressive~\cite{bianchi2023expressiv_pooling} as the Weisfeiler-Leman graph isomorphism test~\cite{xu2018how,10.1609/aaai.v33i01.33014602}, we can directly utilise the standard colour-refinement algorithm to obtain topology-refined colours instead of considering specific GNN architectures.
\Cref{fig:qrqr} shows that colour-refinement first leads to an increase in $\Lambda$ and a decrease in $\Gamma$, suggesting that it produces more colours based on the topology.

To apply our colouring-quality measure to empirical data, we map their feature vectors to discrete colours.
We assign the same colour to any two nodes $v_i, v_j$ whose feature vectors' cosine similarity is at least $\tau$, that is, $\operatorname{sim}_\text{cos}(x_i, x_j) \ge \tau \Rightarrow \zeta\left(v_i\right) = \zeta\left(v_j\right)$.
This induces a transitive relationship: nodes $v_i, v_k$ with $\operatorname{sim}_\text{cos}\left(x_i, x_k\right) < \tau$ receive the same colour when they are connected by a path of sufficiently similar nodes.
Selecting the optimal resolution parameter $\tau$ involves optimising $Q$ and requires a dataset-specific trade-off: a lower $\tau$ results in fewer colours, increasing $\Lambda$, whereas a higher $\tau$ yields more colours, increasing $\Gamma$.
\Cref{fig:qrqr}d shows this increase in relative colour count with respect to an increasing $\tau$ for different PEs.

\section{Evaluation}

\begin{table*}[!t]
    \centering
    \caption{
    Average colouring quality $\overline{Q}$ for empirical, random (same, mixed, distinct), and positional encoding node features for different numbers of colour-refinement (CR) iterations with respect to spectral partitions on the Mutag dataset. $\tau$ is the optimal resolution threshold, $\overline{k/N}$ the resulting colour-to-node ratio before colour refinement, and NMI the agreement between node classes obtained with a GCN and MinCut objective and the respective spectral clusters of the corresponding graphs' adjacency matrices. 
    }
    \resizebox{\linewidth}{!}{
    \begin{tabular}{l|cccc|cccc|cccc|cccc|cccc}
        \toprule
          & \multicolumn{4}{c|}{CR-0} & \multicolumn{4}{c|}{CR-1} &  \multicolumn{4}{c|}{CR-2} &  \multicolumn{4}{c|}{CR-3} & \multicolumn{4}{c}{CR-$\infty$} \\
           & $\overline{Q}$ & \gr{NMI} & $\tau$ & $\overline{k/N}$  & $\overline{Q}$ & \gr{NMI} & $\tau$ & $\overline{k/N}$ & $\overline{Q}$ & \gr{NMI} & $\tau$ & $\overline{k/N}$ & $\overline{Q}$ & \gr{NMI} & $\tau$ & $\overline{k/N}$ & $\overline{Q}$ & \gr{NMI} &$\tau$ & $\overline{k/N}$\\
         \midrule
         Empiric. Features & 0.34 & \gr{0.00} & - & 0.19 & 0.53 & \gr{0.40} & - & 0.19 & 0.47 & \gr{0.52} & - & 0.19 & 0.26 & \gr{0.53} & - & 0.19 & 0.13 & \gr{0.32} & - & 0.19     \\
         \midrule
         Random: same      &  0.00 & \gr{0.00} &  - &  0.06 &  0.14 & \gr{0.00} &  - &  0.06 &  0.50 & \gr{0.00} &  - &  0.06 &  \textbf{0.39} & \gr{0.00} &  - &  0.06 &  0.29 & \gr{0.53} &  - &  0.06  \\
         Random: mixed     &  0.31 & \gr{0.00} &  - &  0.52 &  0.28 & \gr{0.00} &  - &  0.52 &  0.14 & \gr{0.67} &  - &  0.52 &  0.14 & \gr{0.68} &  - &  0.52 &  0.14 & \gr{0.15}&  - &  0.52  \\
         Randim: distinct  &  0.00 & \gr{0.00} &  - &  1.00 &  0.32 & \gr{0.66} &  - &  1.00 &  0.32 & \gr{0.67} &  - &  1.00 &  0.32 & \gr{0.68} &  - &  1.00 &  0.32 & \gr{\textbf{0.65}}&  - &  1.00  \\
         \midrule
         Random Walk PE   &  0.37 & \gr{0.24} &  1.00 &  0.86 &  0.37 & \gr{0.39} &  1.00 &  0.86 &  0.50 & \gr{0.41} &  0.00 &  0.06 &  \textbf{0.39} & \gr{0.60} &  0.00 &  0.06 &  0.29 & \gr{0.54} &  0.00 &  0.06 \\
         Laplacian PE     &  \textbf{0.79} & \gr{\textbf{0.72}} &  0.48 &  0.45 &  \textbf{0.72} & \gr{\textbf{0.67}} &  0.25 &  0.24 &  \textbf{0.63} & \gr{\textbf{0.74}} &  0.12 &  0.14 &  \textbf{0.39} & \gr{\textbf{0.69}} &  0.00 &  0.06 &  \textbf{0.34} & \gr{0.54} &  0.99 &  0.96 \\
         GPSE             &  0.71 & \gr{0.00} &  0.89 &  0.60 &  0.68 & \gr{0.39} &  0.83 &  0.39 &  0.61 & \gr{0.48} &  0.64 &  0.11 &  \textbf{0.39} & \gr{0.60} &  0.00 &  0.06 &  \textbf{0.34} & \gr{0.43} &  0.96 &  0.77 \\
         Node2Vec PE      &  0.72 & \gr{0.00} &  0.71 &  0.41 &  0.58 & \gr{0.64} &  0.62 &  0.20 &  0.59 & \gr{0.61} &  0.52 &  0.09 &  \textbf{0.39} & \gr{0.64} &  0.00 &  0.06 &  0.32 & \gr{0.02} &  1.00 &  1.00 \\
         \bottomrule
    \end{tabular}
    }
    \label{tab:qrdata}
    \vspace{-1em}
\end{table*}

\label{sec:guide}
The question our colouring-quality measure $Q$ aims to answer is: how well do the given node features identify certain node groups using community-based pooling?
We consider empirical, random, and positional encodings as features for graph classification tasks on standard datasets and apply $\{1,2,3,4,\infty\}$ colour-refinement iterations, where $\infty$ denotes ``until convergence''.
For random features, we consider three cases: (1) \textit{same} assigns all nodes the same features, (2) \textit{distinct} assigns all nodes distinct features, and (3) \textit{mixed} assigns distinct features to half of the nodes, the other half's features are drawn, with replacement, from the first half's features.
For positional encodings, we use Laplacian~\cite{pe_laplacian}, GPSE~\cite{pe_gpse}, Random Walk~\cite{pe_rw}, and node2vec~\cite{10.1145/2939672.2939754}.
Spectral clustering of the graphs' adjacency matrices determines target partitions.
We generalise $Q$ for datasets with multiple graphs by (1) taking an 80/20 seen/unseen split, (2) computing the average colouring validity $\overline{\Gamma}$ for all graphs, (3) consider all seen-unseen graph pairs to determine $\overline{\Lambda}$, the fraction of matched unseen graphs, and (4) set the average colouring quality $\overline{Q}$ to be the minimum of $\overline{\Gamma}$ and $\overline{\Lambda}$ (details in \Cref{sec:variants}).
Our code is available on GitHub\footnote{\url{https://github.com/jvpichowski-research/2026-features-for-graph-pooling}}.

\Cref{tab:qrdata} shows the results of our systematic evaluation on the Mutag dataset; \Cref{sec:quality} contains results for further datasets.
We find that PEs, especially Laplacian and node2vec, generally yield better $Q$-scores than empirical or random features, likely because PEs are designed to capture graph structure.
\Cref{fig:pe-examples} shows PEs for a graph from the Mutag dataset, reduced to a single dimension via UMAP~\cite{mcinnes2020umapuniformmanifoldapproximation}:
The Laplacian and node2vec PEs capture the graph's community structure better than GPSE and Random Walk PEs, which assign the same colour to nodes far apart in the graph, suggesting that Laplacian and node2vec PEs are better suited for pooling applications.
We also find that a small number of colour-refinement iterations can be useful when using empirical or random features, but is detrimental to PEs in most cases, a phenomenon likely explained by the over-smoothing of node features.
In addition, we use a GCN with a MinCut objective and $\{0,1,2,3,\infty\}$ layers to classify nodes, and use NMI to compare the resulting classes with the communities obtained from the respective graph's adjacency matrix via spectral clustering.
We choose a GCN because it also adopts a spectral perspective \cite{kipf2017semisupervised}.
The NMI values reported in \Cref{tab:qrdata} and \Cref{sec:quality} align with our $Q$-score and confirm that Laplacian and node2vec PEs are the most useful features.

\paragraph{Experimental Reevaluation of Pooling Operators.}
\begin{wrapfigure}[10]{r}{0.35\textwidth}
    \centering
    \vspace*{-1.5\baselineskip}
    \includegraphics[width=\linewidth]{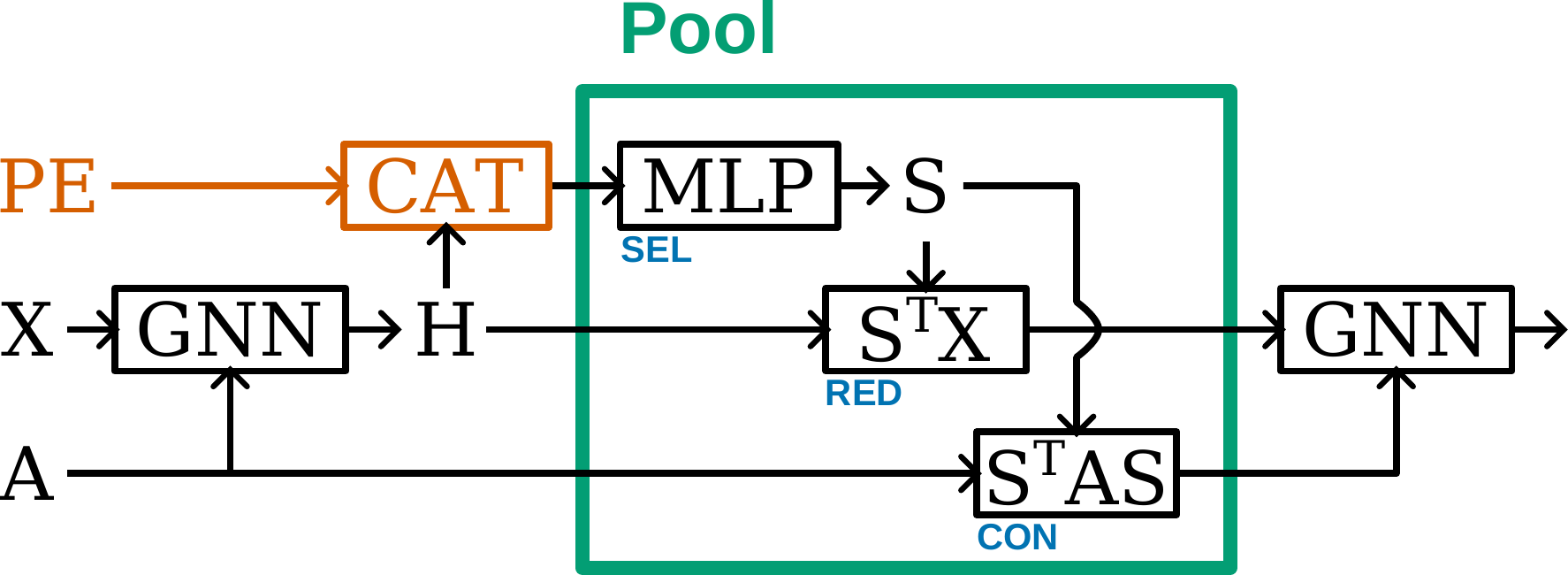}
    \caption{Block diagram of our pooling setup including positional encodings (PEs) in terms of the \textsc{SEL-RED-CON} framework \cite{sel-red-con}.}
    \label{fig:setup}
\end{wrapfigure}
To ensure that the implicit assumption behind pooling---that node features align with topological clusters---is met, we incorporate positional encodings into the pooling setup.
\Cref{fig:setup} illustrates the setup, with further details in \Cref{alg:setup}, \Cref{sec:setup}.
The improvement over previous pooling setups is that, in addition to the intermediate node representations $\mathbf{H}$, the \textsc{SEL} operation also has access to PEs that, by construction, carry information about the graph's topology.
We do not apply message passing to the PEs, as we found colour refinement generally degrades PEs' colouring quality (\Cref{tab:qrdata}).
We test five pooling methods, namely DiffPool \cite{NEURIPS2018_e77dbaf6}, DMoN \cite{JMLR:v24:20-998}, JBPool \cite{Bianchi_2023_jb}, MDL-Pool \cite{mdlpool}, and MinCut \cite{pmlr-v119-bianchi20a}, each 
with and without positional encodings, and compare their performance against the no-pool baseline, which embeds graphs by taking a global maximum of the nodes' features.
\Cref{fig:main-results} visualises our results (details in \Cref{sec:main-tab}, \Cref{tab:main}).
\begin{figure}
    \includegraphics[width=\linewidth]{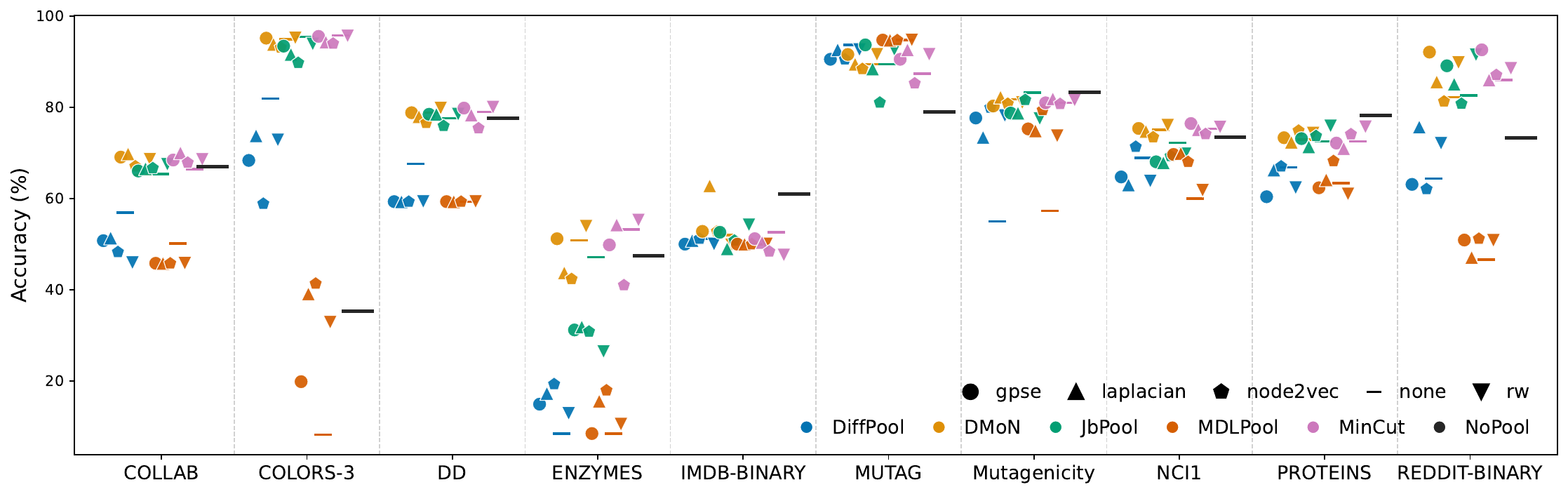}
    \caption{Downstream graph-classification performance for five pooling methods, with and without positional encodings, and no-pool.}
    \label{fig:main-results}
\end{figure}

\begin{wrapfigure}[16]{r}{0.42\textwidth}
    \vspace*{-.5\baselineskip}
    \centering
    \begin{overpic}[width=.45\linewidth]{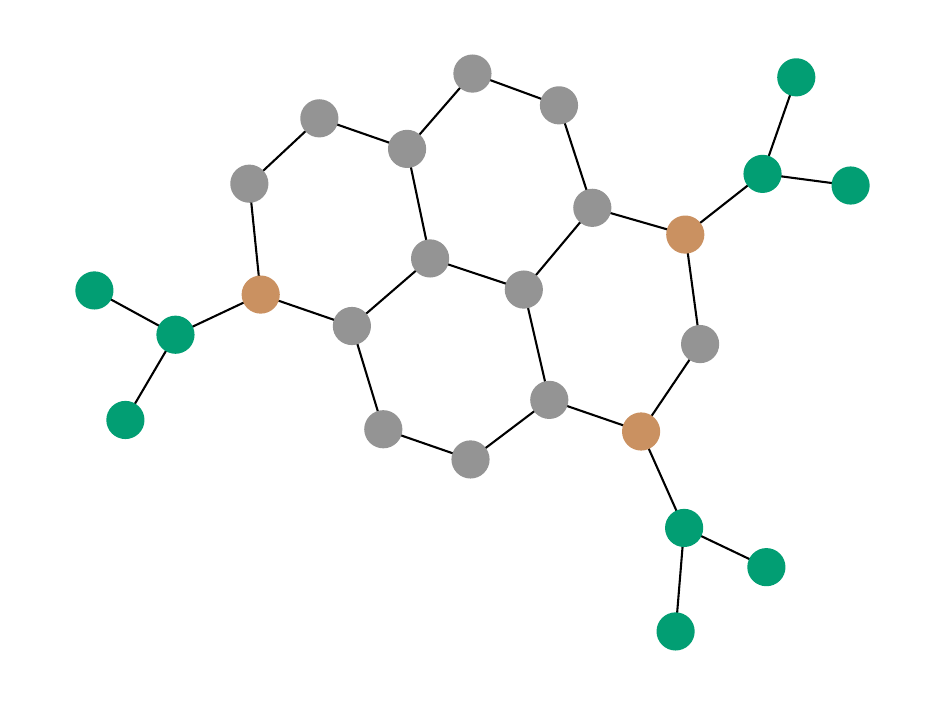}
        \put(0,60){\textbf{(a)}}
    \end{overpic}
    \quad
    \begin{overpic}[width=.45\linewidth]{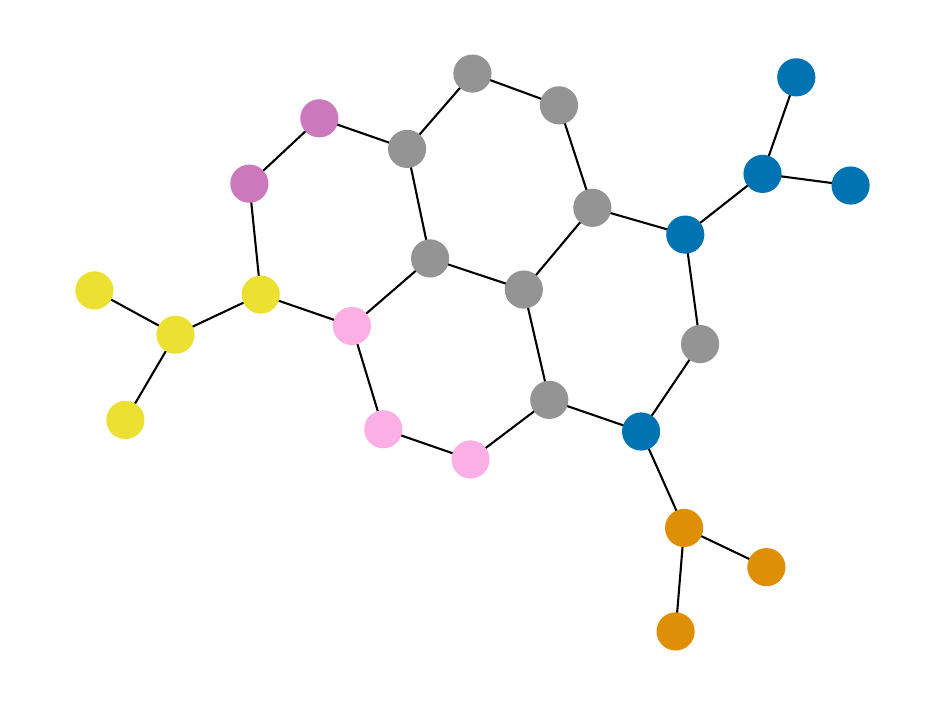}
        \put(0,60){\textbf{(b)}}
    \end{overpic}
    \quad
    \begin{overpic}[width=.45\linewidth]{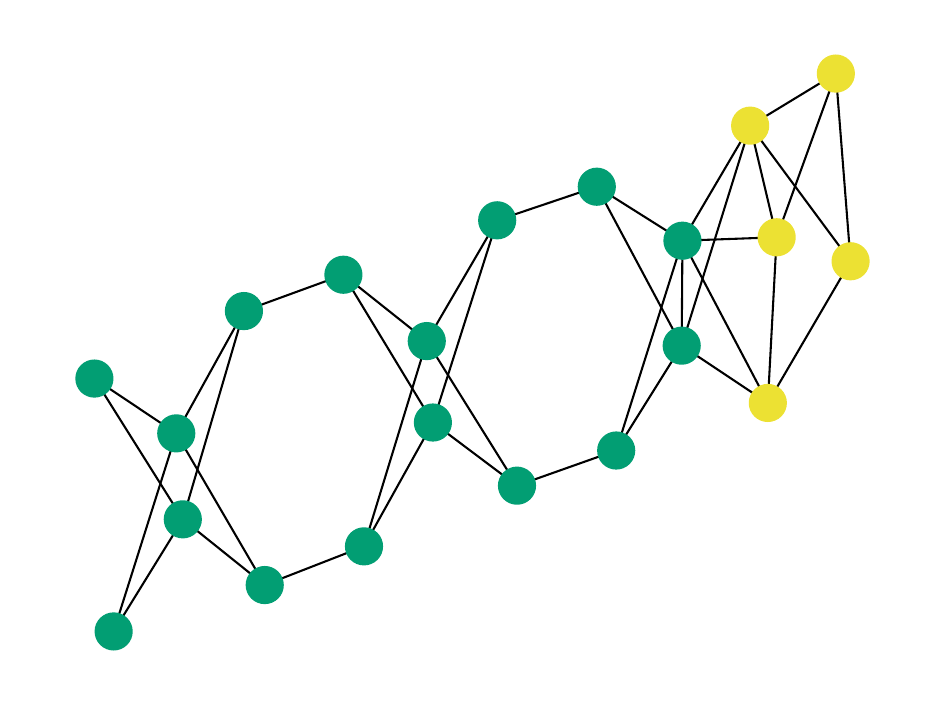}
        \put(0,55){\textbf{(c)}}
    \end{overpic}
    \quad
    \begin{overpic}[width=.45\linewidth]{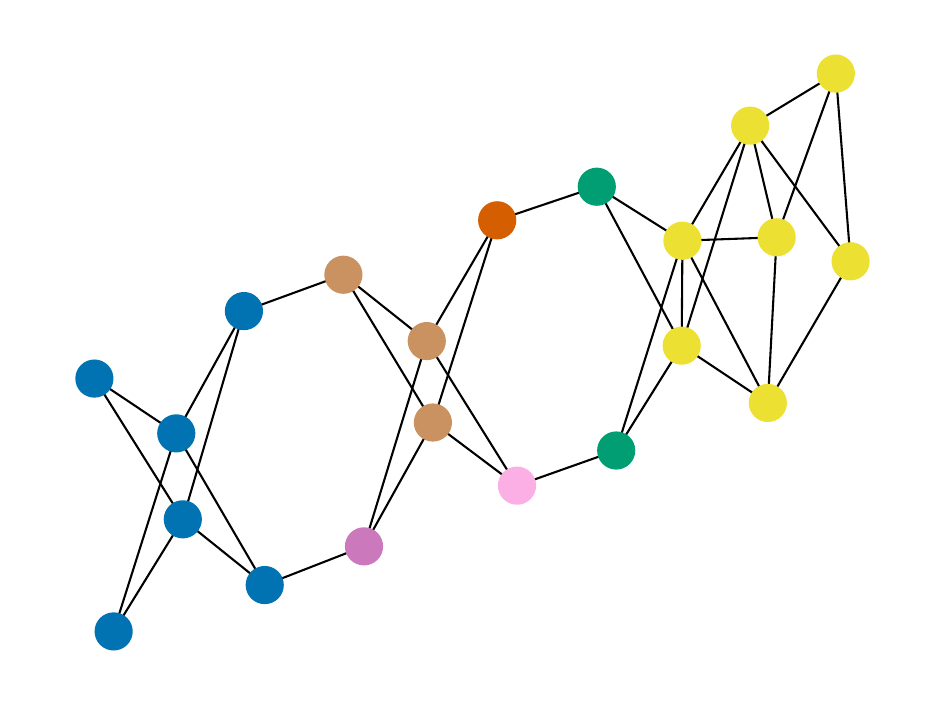}
        \put(0,55){\textbf{(d)}}
    \end{overpic}
    \vspace*{-.5\baselineskip}
    \caption{Improved group assignments obtained with Laplacian PEs.
    \textbf{(a)} \textsc{Mutag}, no PEs
    \textbf{(b)} \textsc{Mutag}, with PEs
    \textbf{(c)} \textsc{Proteins}, no PEs
    \textbf{(d)} \textsc{Proteins}, with PEs.
    }
    \label{fig:learned-assignments}
    \vspace{-1em}
\end{wrapfigure}
Our experimental results indicate a connection between node features, graph topology, and pooling---including PEs in the pooling setup can improve downstream performance.
However, whether community-based pooling performs better or worse than the no-pool baseline depends on the dataset: 
For example, on the \textsc{Reddit-B} and \textsc{Mutag} datasets, pooling without PEs already outperforms the no-pool baseline; including PEs further improves the downstream performance in most cases.
\Cref{fig:learned-assignments} illustrates cases where PEs enable discovering clusters that better align with our intuition of topological clusters.
In contrast, for the \textsc{Proteins} and \textsc{IMDB-B} datasets, no-pool performs better than pooling without PEs; while using PEs for pooling improves pooling, it also often leads to weaker downstream performance than not using PEs.
Consequently, whether to apply pooling to a dataset should be carefully considered.
While previous works blame poor graph classification performance on shortcomings of pooling methods, our results indicate that pooling suitability is an inherent property of datasets.

\section{Conclusion, Limitations, and Future Work}
Community-based deep graph pooling combines the graph topology $\mathbf{A}$ with node features $\mathbf{X}$ to jointly learn communities for coarse-graining the graph, and is often used in graph classification tasks.
We argued that, for this to function, features and topology must ``align'': in \Cref{theorem:align,theorem:injective,theorem:deterministic}, we formulated simple conditions that enable or prevent identifying the desired communities.
Based on these conditions, we developed a colouring-based framework and quality measure (\Cref{eqn:quality}) for assessing the suitability of features for identifying given communities.
We applied our feature quality measure to empirical graph datasets and found that alignment between features $\mathbf{X}$ and the topology $\mathbf{A}$ cannot generally be assumed; often, a limited number of colour-refinement iterations can improve the situation.
However, we found that positional encodings (PEs) are considerably more aligned with the graph topology and, thus, more suitable for community-based graph pooling (see \Cref{tab:qrdata} and \Cref{sec:quality}).
Motivated by these results, we incorporated positional encodings into graph pooling for a downstream graph classification task (\Cref{fig:setup}), and found that PEs often improve downstream task performance.
However, we also found that in some datasets the opposite is true: using positional encodings does not guarantee improved downstream task performance (see \Cref{fig:main-results} and \Cref{sec:main-tab}).
Whereas previous work often blames limited pooling performance on the shortcomings of pooling operators, our results suggest that pooling suitability is an inherent property of a dataset.

\vspace*{-.1\baselineskip}
\textbf{Limitations.}
\label{sec:limitations}
Our feature-quality measure is agnostic to the precise definition of node groups and is, in principle, applicable to pooling approaches that are not community-based, such as ranking-based pooling.
However, the different semantics of groups, such as \textit{important} vs.\ \textit{not important}, differ fundamentally from those of communities, and are not immediately compatible with the reference-partition-based evaluation concept we adopted.
For our evaluation, we assumed that the graphs' properties in the datasets are informative for predicting their class labels.
However, poor pooling performance may be due to low-quality class labels and the interplay between identified groups and class labels; a rigorous evaluation of class label quality is beyond the scope of this work.

\vspace*{-.1\baselineskip}
\textbf{Future Work.}
Our work highlights important gaps in current graph-pooling setups, calling for follow-up studies:
First, GNN-based graph clustering relies on a topology-refined version of node features.
However, the extent to which colour refinement incorporates topological information is limited, so that GNNs essentially ignore precise topological patterns when constructing clusters.
Developing clustering approaches for GNNs that incorporate topology more directly is a promising direction for reducing reliance on high-quality features.
Alternatively, node features could be incorporated more into pooling operators.
Second, we found that, for some datasets, GNN-based pooling for graph classification performs worse than the simple no-pool baseline.
We suspect this may be due to a mismatch between the supervision labels and the graphs' topologies and node features.
However, a systematic study of the datasets' properties is required to pinpoint the precise reasons.
Third, we anticipate that our colouring-based framework can be generalised for assessing the suitability of node features and graph topology for tasks beyond graph classification.

\begin{ack}
    JP and IS acknowledge funding by the German Ministry of  Research, Technology and Space (BMFTR) under grant agreement No. 01IS24072A (COMFORT).
    This work was partially supported by the Wallenberg AI, Autonomous Systems and Software Program (WASP) funded by the Knut and Alice Wallenberg Foundation.
\end{ack}


\clearpage
\appendix

\section{Notes on the Quality Scores}
\label{sec:variants}

\subsection{Trivial solutions for optimal scores}
\begin{figure}[h]
    \vspace*{1\baselineskip}
    \centering
    \begin{overpic}[width=0.9\linewidth]{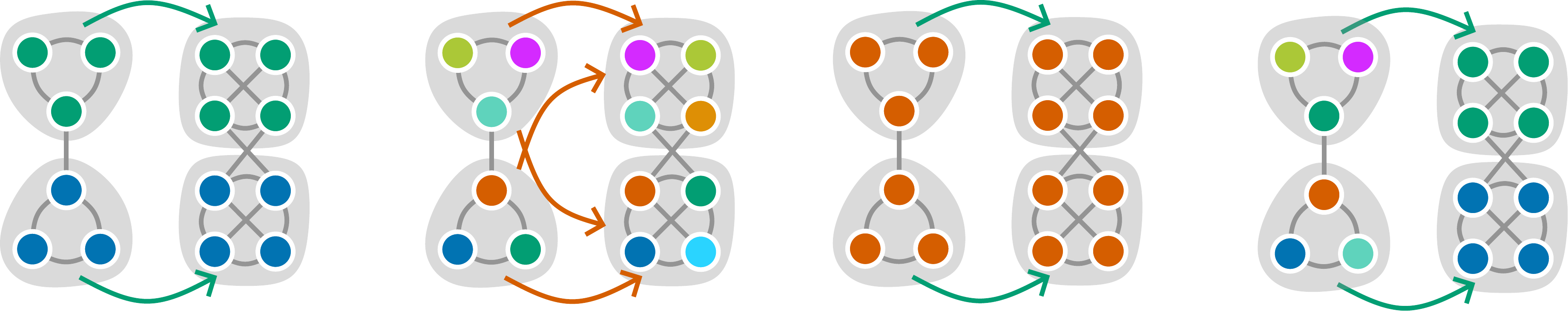}
        \put(-3,19){\textbf{(a)}}
        \put(4.5,21.5){\tiny$\left\{
            \tikz \filldraw[draw=fadedgreen,fill=fadedgreen] (0,0) circle (2pt);
        \right\} \supseteq \left\{
            \tikz \filldraw[draw=fadedgreen,fill=fadedgreen] (0,0) circle (2pt);
        \right\}$}
        \put(4.5,-2.5){\tiny$\left\{
            \tikz \filldraw[draw=windowsblue,fill=windowsblue] (0,0) circle (2pt);
        \right\} \supseteq \left\{
            \tikz \filldraw[draw=windowsblue,fill=windowsblue] (0,0) circle (2pt);
        \right\}$}
        
        \put(23,19){\textbf{(b)}}
        \put(26.5,21.5){\tiny$\left\{
            \tikz \filldraw[draw=lightbluea,fill=lightbluea] (0,0) circle (2pt);,
            \tikz \filldraw[draw=pinka,fill=pinka] (0,0) circle (2pt);,
            \tikz \filldraw[draw=lightgreena,fill=lightgreena] (0,0) circle (2pt);
        \right\} \not\supseteq \left\{
            \tikz \filldraw[draw=lightbluea,fill=lightbluea] (0,0) circle (2pt);,
            \tikz \filldraw[draw=dustyorange,fill=dustyorange] (0,0) circle (2pt);,
            \tikz \filldraw[draw=pinka,fill=pinka] (0,0) circle (2pt);,
            \tikz \filldraw[draw=lightgreena,fill=lightgreena] (0,0) circle (2pt);
        \right\}$}
        \put(26.5,-2.5){\tiny$\left\{
            \tikz \filldraw[draw=windowsblue,fill=windowsblue] (0,0) circle (2pt);,
            \tikz \filldraw[draw=palered,fill=palered] (0,0) circle (2pt);,
            \tikz \filldraw[draw=fadedgreen,fill=fadedgreen] (0,0) circle (2pt);
        \right\} \not\supseteq \left\{
            \tikz \filldraw[draw=windowsblue,fill=windowsblue] (0,0) circle (2pt);,
            \tikz \filldraw[draw=lightblueb,fill=lightblueb] (0,0) circle (2pt);,
            \tikz \filldraw[draw=palered,fill=palered] (0,0) circle (2pt);,
            \tikz \filldraw[draw=fadedgreen,fill=fadedgreen] (0,0) circle (2pt);
        \right\}$}
        \put(36,7.5){\tiny$\not\supseteq$}
        \put(36,12){\tiny$\not\supseteq$}

        \put(50,19){\textbf{(c)}}
        \put(57,21.5){\tiny$\left\{
            \tikz \filldraw[draw=palered,fill=palered] (0,0) circle (2pt);
        \right\} \supseteq \left\{
            \tikz \filldraw[draw=palered,fill=palered] (0,0) circle (2pt);
        \right\}$}
        \put(57,-2.5){\tiny$\left\{
            \tikz \filldraw[draw=palered,fill=palered] (0,0) circle (2pt);
        \right\} \supseteq \left\{
            \tikz \filldraw[draw=palered,fill=palered] (0,0) circle (2pt);
        \right\}$}
        
        \put(78,19){\textbf{(d)}}
        \put(82.5,21.5){\tiny$\left\{
            \tikz \filldraw[draw=lightbluea,fill=lightbluea] (0,0) circle (2pt);,
            \tikz \filldraw[draw=pinka,fill=pinka] (0,0) circle (2pt);,
            \tikz \filldraw[draw=fadedgreen,fill=fadedgreen] (0,0) circle (2pt);
        \right\} \supseteq \left\{
            \tikz \filldraw[draw=fadedgreen,fill=fadedgreen] (0,0) circle (2pt);
        \right\}$}
        \put(82.5,-2.5){\tiny$\left\{
            \tikz \filldraw[draw=windowsblue,fill=windowsblue] (0,0) circle (2pt);,
            \tikz \filldraw[draw=palered,fill=palered] (0,0) circle (2pt);,
            \tikz \filldraw[draw=lightbluea,fill=lightbluea] (0,0) circle (2pt);
        \right\} \supseteq \left\{
            \tikz \filldraw[draw=windowsblue,fill=windowsblue] (0,0) circle (2pt);
        \right\}$}
    \end{overpic}
    \vspace*{1\baselineskip}
    \caption{The left graph is seen and mapped to the right graph that is unseen during training.\\
    \textbf{(a)} Optimal colourings: The colours are valid $\Gamma = 1$ and transferable $\Lambda = 1$, leading to $Q = 1$.\\
    \textbf{(b)} Valid colourings: The colours are valid $\Gamma = 1$ but not transferable $\Lambda = 0$, leading to $Q = 0$.\\
    \textbf{(c)} Transferable colourings: The colours are invalid $\Gamma = 0$ but transferable $\Lambda = 1$, leading to $Q = 0$.\\
    \textbf{(d)} Optimal colourings: The colours are valid $\Gamma = 1$ and transferable $\Lambda = 1$, leading to $Q = 1$.\\
    }
    \label{fig:valid-match-2}
\end{figure}
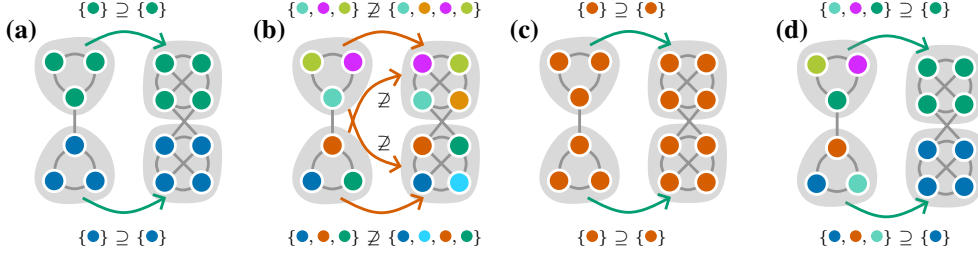

\subsection{Implementation details and Complexity}

\subsubsection{Computation of $\Gamma(\zeta | P)$}
To efficiently compute $\Gamma$ as in \Cref{alg:gamma}, we represent colours, nodes, and groups using integer indices.
The colouring $\zeta$ is stored as a list indexed by node indices, where each entry contains the colour assigned to the corresponding node. 
This representation requires $\mathcal{O}(|V|)$ memory.
The partition $P$ consists of group indices associated with lists of node indices. 
Since each node belongs to at most one group, the total memory consumption of $P$ is also $\mathcal{O}(|V|)$.
In addition, we maintain a list $U$ that maps each colour index to a unique group index. 
This auxiliary structure requires $\mathcal{O}(|C|)$ memory.

The algorithm iterates over all nodes by traversing the groups in $P$, which takes $\mathcal{O}(|V|)$ time in total. 
For each node, its colour is retrieved in constant time using the node index, and the corresponding entry in $U$ is updated with the current group index.
If an entry in $U$ needs to be overwritten for the first time, the algorithm increases a counter variable and sets the entry as 'invalid', indicating that the same colour appears in more than one group and the assignment is therefore invalid.
The algorithm completes by returning the number of invalid colours divided by the number of all colours.

Overall, the runtime of $\Gamma$ is $\mathcal{O}(|V|)$, and its space complexity is $\mathcal{O}(|V| + |C|)$.

\begin{figure}[b!]
\centering
\begin{minipage}{\textwidth}
\begin{algorithm}[H]
\caption{$\Gamma(\zeta|P)$}\label{alg:gamma}
\begin{algorithmic}
\REQUIRE $\zeta, P$
\ENSURE $\Gamma$
\STATE $c \gets \textsc{max}(\zeta)$
\STATE $U \gets \textsc{List}(c)$
\STATE $\text{inv} \gets 0$
\STATE $\textsc{For } g_i \in P \textsc{ Do}$
\STATE $\quad \textsc{For } v_j \in g_i \textsc{ Do}$
\STATE $\quad \quad c_k \gets \zeta[i]$
\STATE $\quad \quad \textsc{If } U[k] = \text{'none'} \textsc{ Then}$
\STATE $\quad \quad \quad U[k] \gets i$
\STATE $\quad \quad \textsc{ElseIf } U[k] = \text{'invalid'} \textsc{ Then}$
\STATE $\quad \quad \quad \textsc{Continue}$
\STATE $\quad \quad \textsc{Else}$
\STATE $\quad \quad \quad U[k] \gets \text{'invalid'}$
\STATE $\quad \quad \quad \text{inv} \gets \text{inv} + 1$
\STATE $\quad \quad \textsc{EndIf}$
\STATE $\quad \textsc{EndFor}$
\STATE $\textsc{EndFor}$
\STATE $\Gamma \gets \text{inv} / c$
\end{algorithmic}
\end{algorithm}
\end{minipage}
\end{figure}

\subsubsection{Computation of $\Lambda\left( \zeta_s, \zeta_u \,|\, P_s, P_u \right)$}
For an efficient computation of $\Lambda$ as in \Cref{alg:lambda}, we first construct an auxiliary list $S$ containing one entry for each seen group. 
Each entry is implemented as a hash set storing the colours present in the corresponding group. 
Assuming a perfect hash function, queries run in amortised $\mathcal{O}(1)$ time, or $\mathcal{O}(1)$ on average in general.
Since there are at most $|V_s|$ seen nodes and each node contributes exactly one colour, the construction of these hash sets takes $\mathcal{O}(|V_s|)$ time and space.

Next, we iterate over all unseen groups in $P_u$ and examine the colours of their nodes. 
Each unseen node is processed exactly once, resulting in at most $\mathcal{O}(|V_u|)$ colour lookups.
For each unseen group, we check whether there exists a seen group that contains all the colours of the unseen group. 
This is done by iterating over all seen groups and querying their hash sets. 
Due to the constant-time complexity of hash set lookups, checking a single colour against one seen group takes $\mathcal{O}(1)$ time.
Consequently, processing a single unseen node requires at most $\mathcal{O}(|P_s|)$ time.

If no matching seen group is found for an unseen group, the algorithm terminates early and returns~0. 
Otherwise, the procedure completes and returns~1.

In summary, $\Lambda$ runs on average in time $\mathcal{O}(|V_u|,|P_s| + |V_s|)$ and uses $\mathcal{O}(|V_u| + |V_s|)$ space.
If hash sets are replaced by balanced trees of colours, the runtime becomes $\mathcal{O}(|V_u|,|P_s| + |V_s|\log |C_s|)$ due to the increased cost of constructing the data structure, while the space complexity remains unchanged.

\begin{figure}[b!]
\centering
\begin{minipage}{\textwidth}
\begin{algorithm}[H]
\caption{$\Lambda\left( \zeta_s, \zeta_u \,|\, P_s, P_u \right)$}\label{alg:lambda}
\begin{algorithmic}
\REQUIRE $\zeta_s, \zeta_u, P_s, P_u$
\ENSURE $\Lambda$
\STATE $S \gets \textsc{List}(|P_s|)$
\STATE $\textsc{For } g_i \in P_s \textsc{ Do}$
\STATE $\quad \textsc{For } v_j \in g_s \textsc{ Do}$
\STATE $\quad \quad \textsc{InsertIn}(S[i], \zeta_s[j])$
\STATE $\textsc{For } g_i \in P_u \textsc{ Do}$
\STATE $\quad \text{transferable} \gets \textsc{False}$
\STATE $\quad \textsc{For } S_j \in S \textsc{ Do}$
\STATE $\quad \quad \text{match} \gets \textsc{True}$
\STATE $\quad \quad \textsc{For } v_k \in g_i \textsc{ Do}$
\STATE $\quad \quad \quad \textsc{If } \zeta_u[k] \notin S_j \textsc{ Then}$
\STATE $\quad \quad \quad \quad \text{match} \gets \textsc{False}$
\STATE $\quad \quad \quad \quad \textsc{Break}$
\STATE $\quad \quad \quad \textsc{EndIf}$
\STATE $\quad \quad \textsc{EndFor}$
\STATE $\quad \quad \textsc{If } \text{match} = \textsc{False} \textsc{ Then}$
\STATE $\quad \quad \quad \textsc{Continue}$
\STATE $\quad \quad \textsc{EndIf}$
\STATE $\quad \quad \text{transferable} \gets \textsc{True}$
\STATE $\quad \textsc{EndFor}$
\STATE $\quad \textsc{If } \text{transferable} = \textsc{False} \textsc{ Then}$
\STATE $\quad \quad \Lambda \gets 0$
\STATE $\quad \quad \textsc{Exit}$
\STATE $\quad \textsc{EndIf}$
\STATE $\textsc{EndFor}$
\STATE $\Lambda \gets 1$
\end{algorithmic}
\end{algorithm}
\end{minipage}
\end{figure}

\subsubsection{Computation of $Q\left( \zeta_s, \zeta_u \,|\, P_s, P_u \right)$}
To compute $Q$, we compute $\Gamma$ and $\Lambda$, and return their minimum.
Consequently, the average-case runtime for computing $Q$ is $\mathcal{O}(|V_u|,|P_s| + |V_s|)$, and its worst-case runtime is $\mathcal{O}(|V_u|,|P_s| + |V_s|\log |C_s|)$.
The space complexity for computing $Q$ is $\mathcal{O}(|V_u| + |V_s|)$.

\subsection{Applying the scores to multi-graph datasets}
Pooling methods are applied to datasets with multiple train and test graphs.
We extend our scores to this setting by averaging scores across multiple graphs.

We compute $\bar{\Gamma}(\mathcal{G})$ by iterating through all graphs and taking the average:
\begin{align}
    \bar{\Gamma}(\mathcal{G}) = \frac{\sum_{G \in \mathcal{G}} \Gamma(P, \zeta)}{\left| \mathcal{G} \right|}
\end{align}
We compute $\bar{\Lambda}(\mathcal{G}_s, \mathcal{G}_u)$ by iterating through all unseen graphs and determining if there is at least one seen graph whose colouring matches:
\begin{align}
    \bar{\Lambda}(\mathcal{G}_s, \mathcal{G}_u) = \frac{\left|\{ G_u \in \mathcal{G}_u \,|\, \sum_{G_s \in \mathcal{G}_s} \Lambda(\zeta_s, \zeta_u \,|\, P_s, P_u) \geq 1 \}\right|}{\left| \mathcal{G}_u\right|}
\end{align}
We compute $\bar{Q}(\mathcal{G}_s, \mathcal{G}_u)$ as the limiting factor of $\bar{\Gamma}(\mathcal{G}_s)$ and $\bar{\Lambda}(\mathcal{G}_s, \mathcal{G}_u)$:
\begin{align}
    \bar{Q} \left( \mathcal{G}_s, \mathcal{G}_u \right) = \min \left( \bar{\Gamma}\left(  \mathcal{G}_s \cup \mathcal{G}_u \right), \bar{\Lambda}\left( \mathcal{G}_s, \mathcal{G}_u \right) \right)
\end{align}

Applying $\Gamma$ to all graphs yields the same runtime bounds $\mathcal{O}(|V|)$ but now proportional to the aggregate node count $|V^{\textsc{all}}|$ and not the node count of a single graph. 
Analogously, the runtimes of $\Lambda$ and $Q$ now depend on the total numbers of nodes $|V_u^{\textsc{all}}|, |V_s^{\textsc{all}}|$ and groups $|P_s^{\textsc{all}}|$, since $\Lambda$ is evaluated for all groups.

\subsection{Variants of transferability measure}
We investigate three variants of our transferability measure. 
The first variant is used throughout the main body of the paper. 
In the following, we consider alternative definitions of transferability and show the influence on the resulting quality curves.

\clearpage
\subsubsection{Fully transferable graphs}
A graph is considered either fully transferable or non-transferable. An unseen graph is transferable if all of its groups can be mapped to groups in any seen graph. The final score is computed as the ratio of transferable graphs in the unseen set to the total number of graphs in the unseen set (\Cref{fig:random_coloring_v1}).  
\begin{align}
    &\Lambda\left( \zeta_s, \zeta_u \,|\, P_s, P_u\right) = \begin{cases}
        1 \text{ if } \lambda\left( P_s, g_u \,|\, \zeta_s, \zeta_u \right) \geq 1 \,\forall\, g_u \in P_u
        \\
        0 \text{ otherwise}
    \end{cases}\\
    &\bar{\Lambda}(\mathcal{G}_s, \mathcal{G}_u) = \frac{\left|\{ G_u \in \mathcal{G}_u \,|\, \sum_{G_s \in \mathcal{G}_s} \Lambda(\zeta_s, \zeta_u \,|\, P_s, P_u) \geq 1 \}\right|}{\left| \mathcal{G}_u\right|}
\end{align}

\begin{figure*}[h]
\centering
\begin{subfigure}{\textwidth}
    \centering
    \includegraphics[width=0.94\linewidth]{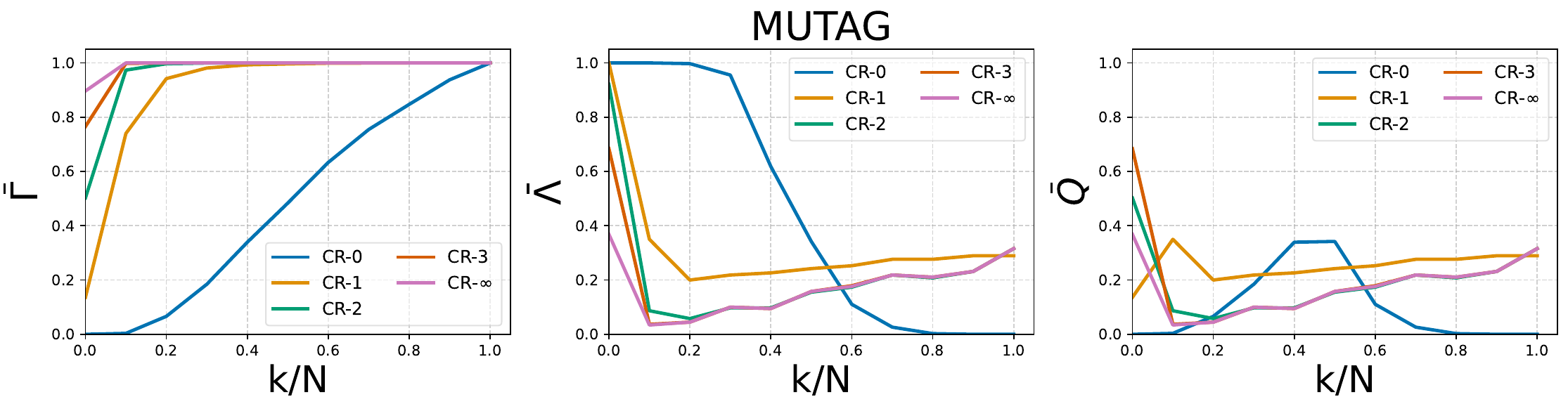}
\end{subfigure}

\begin{subfigure}{\textwidth}
    \centering
    \includegraphics[width=0.94\linewidth]{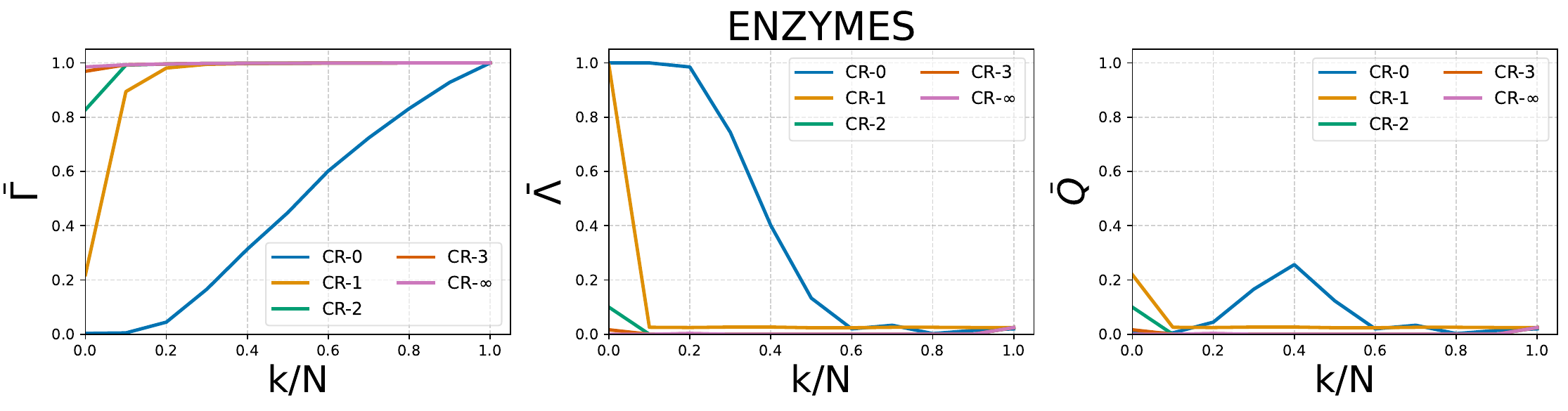}
\end{subfigure}

\begin{subfigure}{\textwidth}
    \centering
    \includegraphics[width=0.94\linewidth]{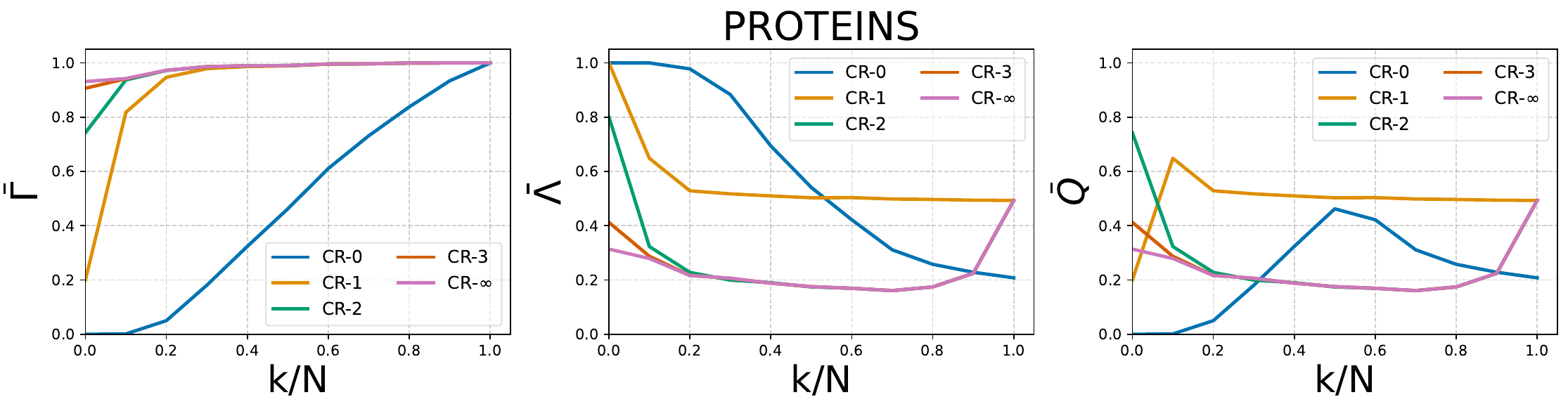}
\end{subfigure}

\begin{subfigure}{\textwidth}
    \centering
    \includegraphics[width=0.94\linewidth]{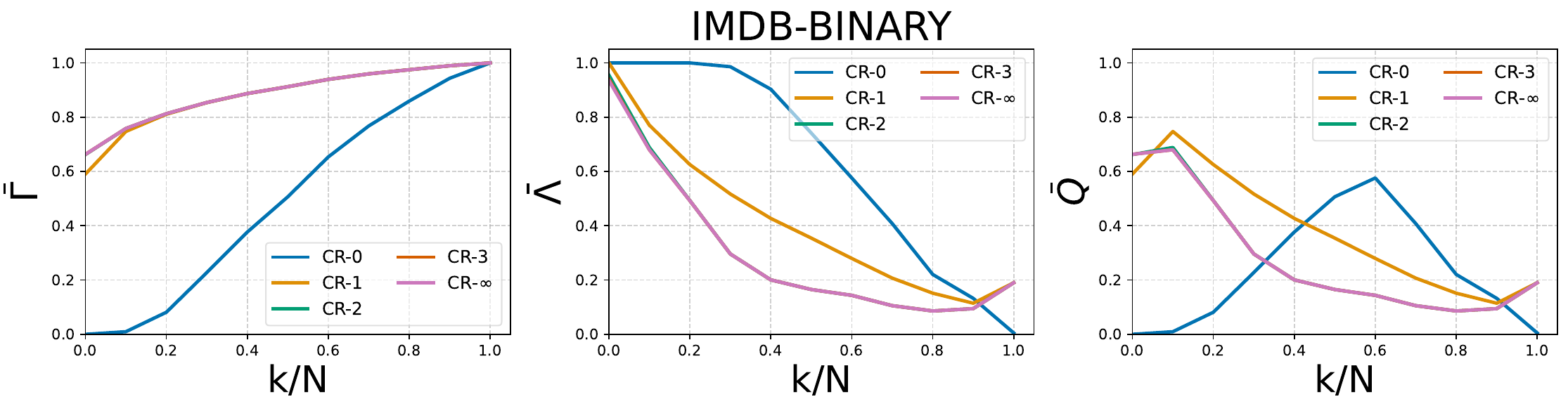}
\end{subfigure}

\begin{subfigure}{\textwidth}
    \centering
    \includegraphics[width=0.94\linewidth]{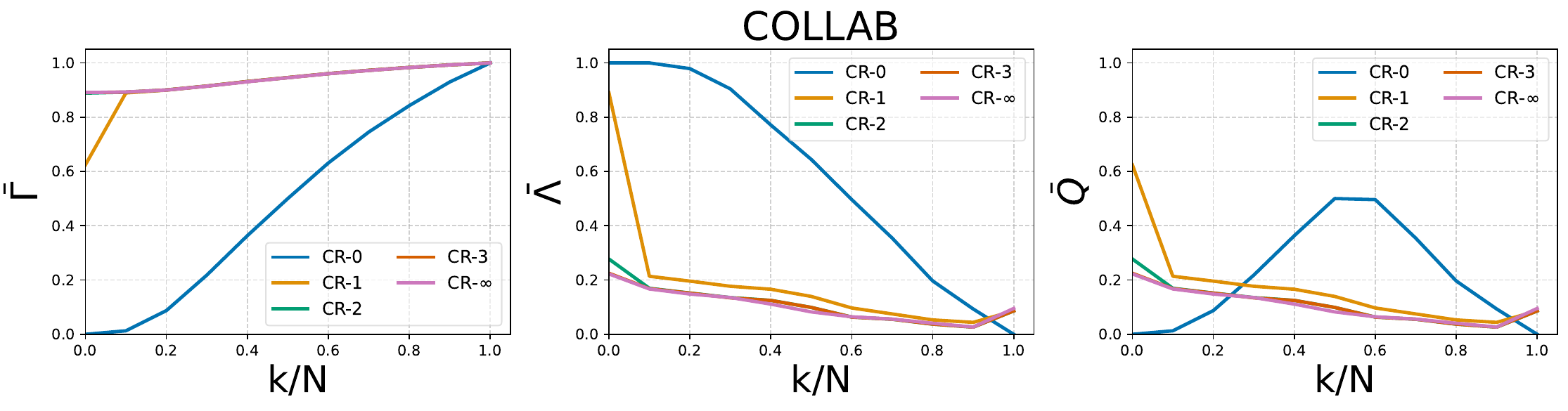}
\end{subfigure}
\caption{Main variant; considering graphs as transferable only if all groups match.}
\label{fig:random_coloring_v1}
\end{figure*}

\clearpage
\subsubsection{Ratio of transferable graphs}
Each unseen graph is assigned a score between 0 and 1, computed as the ratio of transferable groups in the graph to the total number of groups in the graph. 
For each unseen graph, we select the corresponding seen graph that leads to the highest ratio.
Then we compute the average score over the highest ratio of all unseen graphs (\Cref{fig:random_coloring_v2}).
\begin{align}
    & \Lambda\left( \zeta_s, \zeta_u \,|\, P_s, P_u\right) = \frac{ \left| \left\{ g_u \in P_u \,|\, \lambda\left( P_s, g_u \,|\, \zeta_s, \zeta_u \right) \geq 1 \right\} \right|}{\left|P_u\right|}\\
    &\bar{\Lambda}(\mathcal{G}_s, \mathcal{G}_u) = \frac{\left|\{ G_u \in \mathcal{G}_u \,|\, \max_{G_s \in \mathcal{G}_s} \Lambda(\zeta_s, \zeta_u \,|\, P_s, P_u)  \}\right|}{\left| \mathcal{G}_u\right|}
\end{align}

\begin{figure*}[h]
\centering
\begin{subfigure}{\textwidth}
    \centering
    \includegraphics[width=0.91\textwidth]{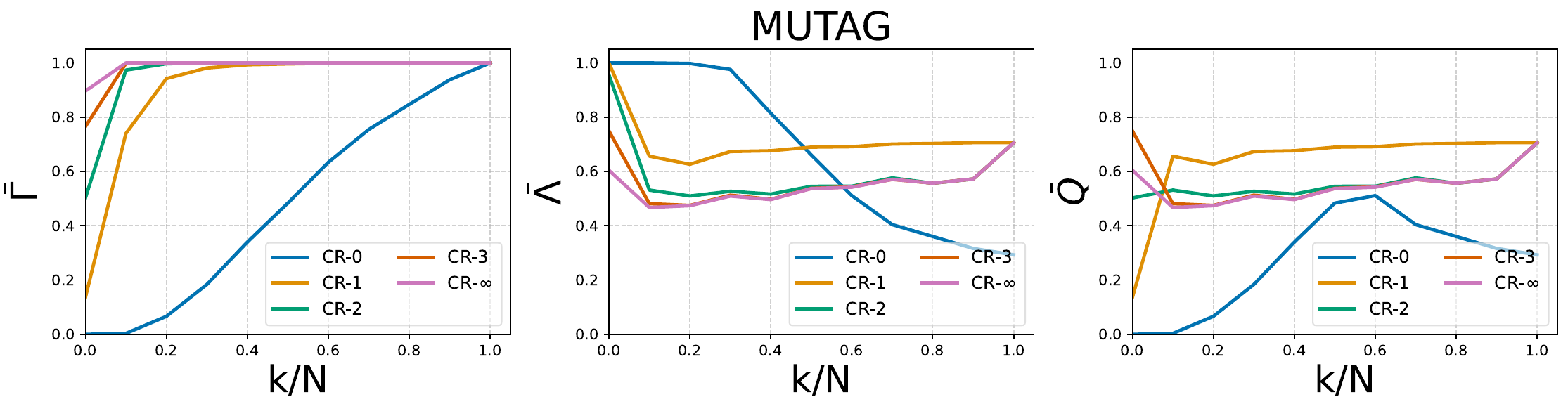}
\end{subfigure}

\begin{subfigure}{\textwidth}
    \centering
    \includegraphics[width=0.91\textwidth]{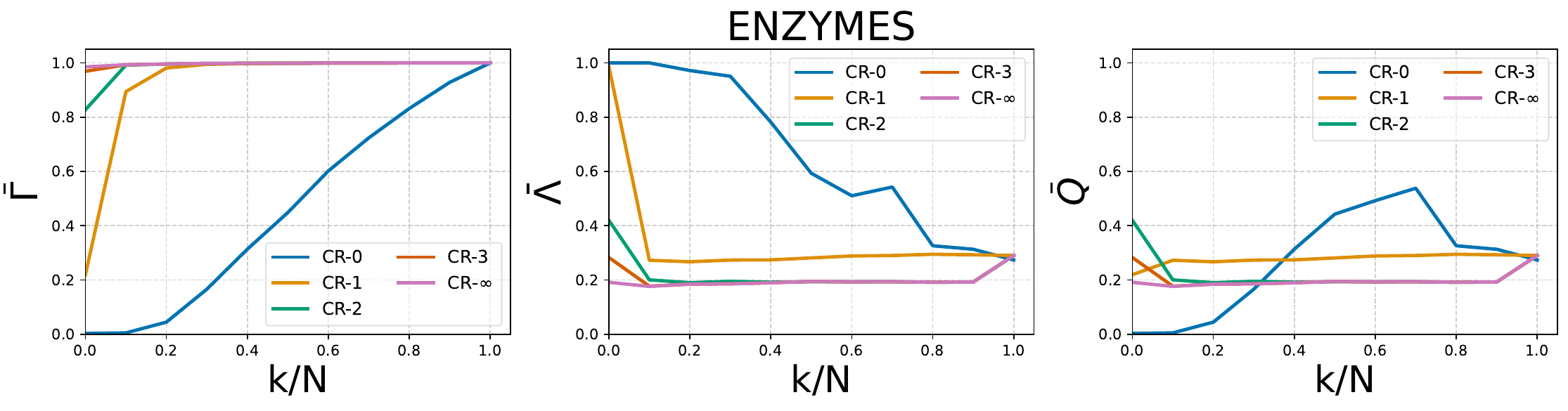}
\end{subfigure}

\begin{subfigure}{\textwidth}
    \centering
    \includegraphics[width=0.91\textwidth]{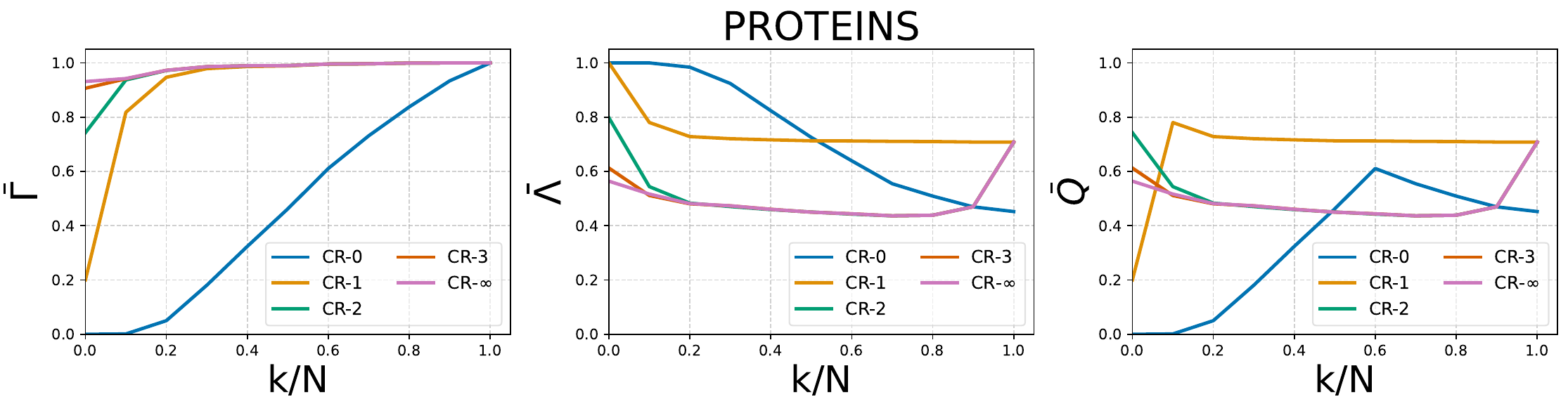}
\end{subfigure}

\begin{subfigure}{\textwidth}
    \centering
    \includegraphics[width=0.91\textwidth]{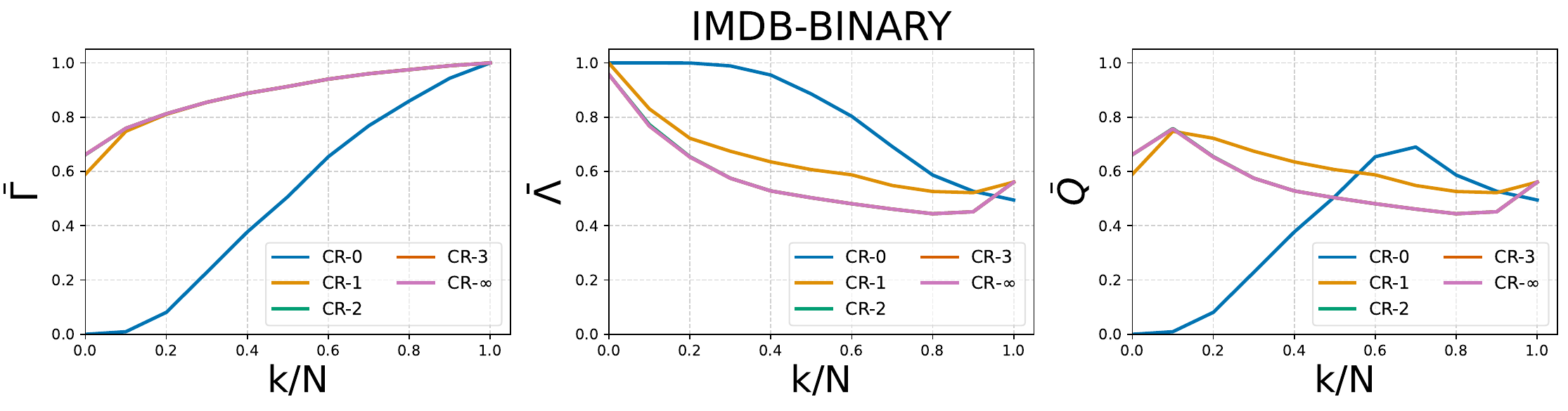}
\end{subfigure}

\begin{subfigure}{\textwidth}
    \centering
    \includegraphics[width=0.91\textwidth]{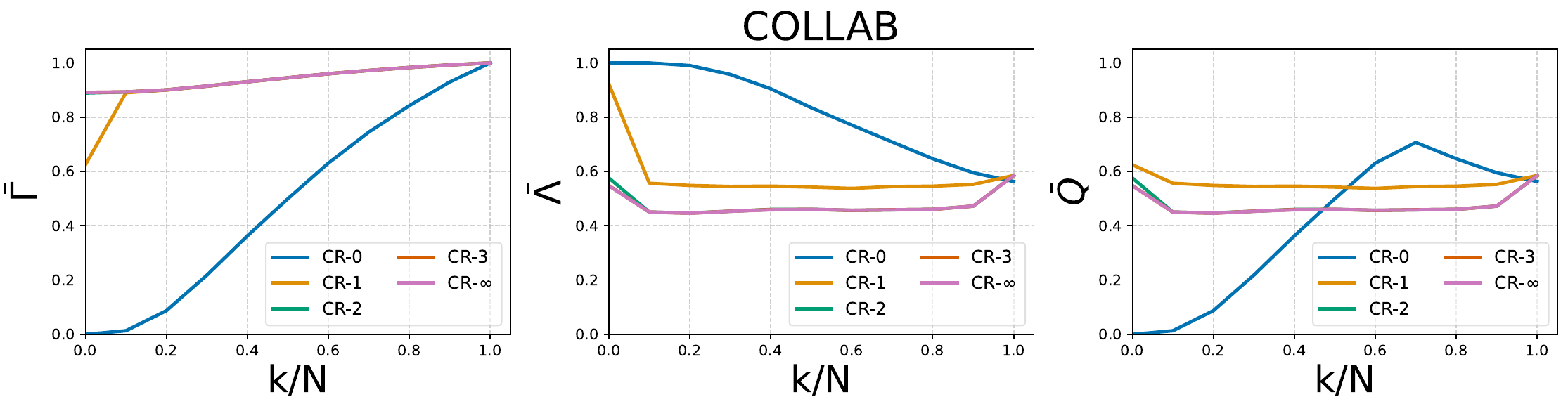}
\end{subfigure}
\caption{Ratio variant; considering the ratio to which a graph matches by counting the relative number of matching groups.}
\label{fig:random_coloring_v2}
\end{figure*}

\clearpage
\subsubsection{Transferable groups}
Each group in the unseen set is assigned a score of either 0 or 1, indicating whether it can be mapped to at least one group in the seen set. The final score is computed as the average over all groups in the unseen set. Here we calculate how many groups are transferable, not the graphs themselves (\Cref{fig:random_coloring_v3}).
\begin{align}
    &\Lambda\left( \zeta_u \,|\, \mathcal{G}_s, g_u\right) = \max_{g_s \in G_s \forall G_s \in \mathcal{G}_s} \mathbf{1}_{\zeta_s, \zeta_u}\left(g_s, g_u\right)\\
    &\bar{\Lambda}(\mathcal{G}_s, \mathcal{G}_u) = \frac{\sum_{g_u \in G_u \forall G_u \in \mathcal{G}_u} \Lambda\left( \zeta_u \,|\, \mathcal{G}_s, g_u\right)}{\sum_{G_u \in \mathcal{G}_u} |P_u|}
\end{align}

\begin{figure*}[h]
\centering
\begin{subfigure}{\textwidth}
    \centering
    \includegraphics[width=0.91\textwidth]{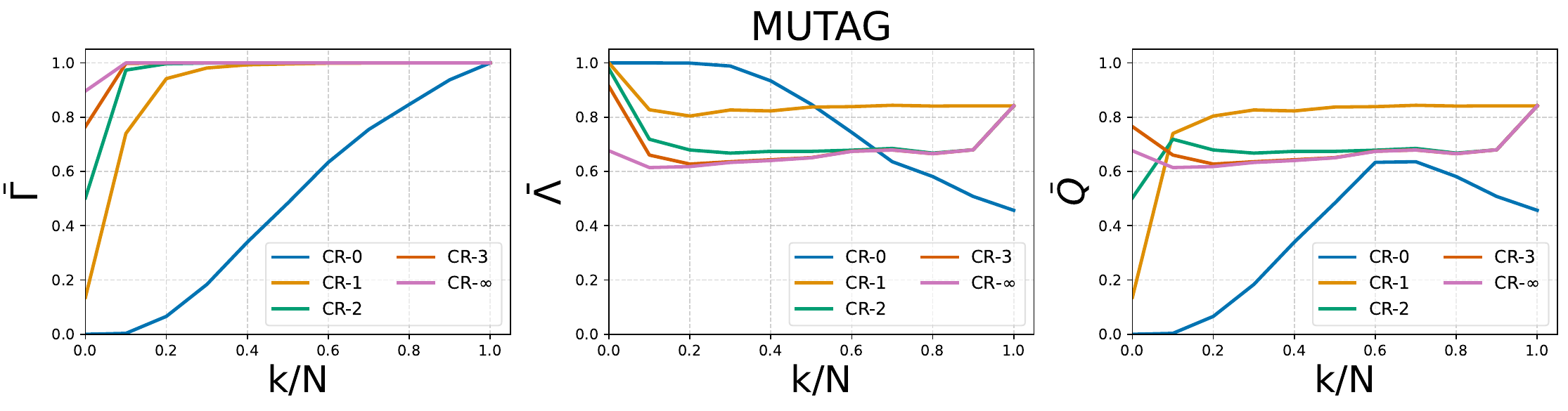}
\end{subfigure}

\begin{subfigure}{\textwidth}
    \centering
    \includegraphics[width=0.91\textwidth]{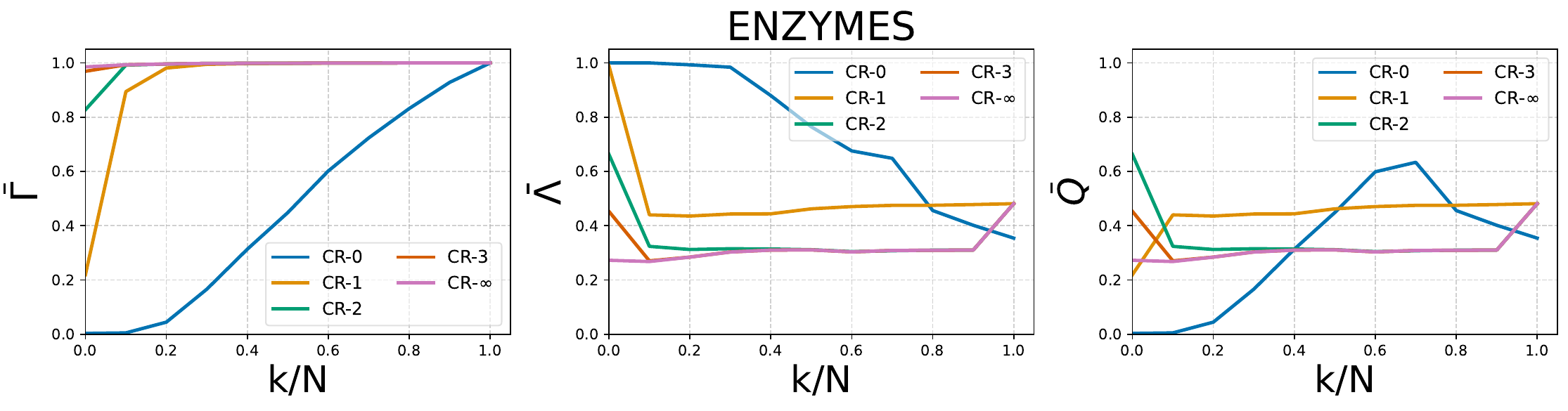}
\end{subfigure}

\begin{subfigure}{\textwidth}
    \centering
    \includegraphics[width=0.91\textwidth]{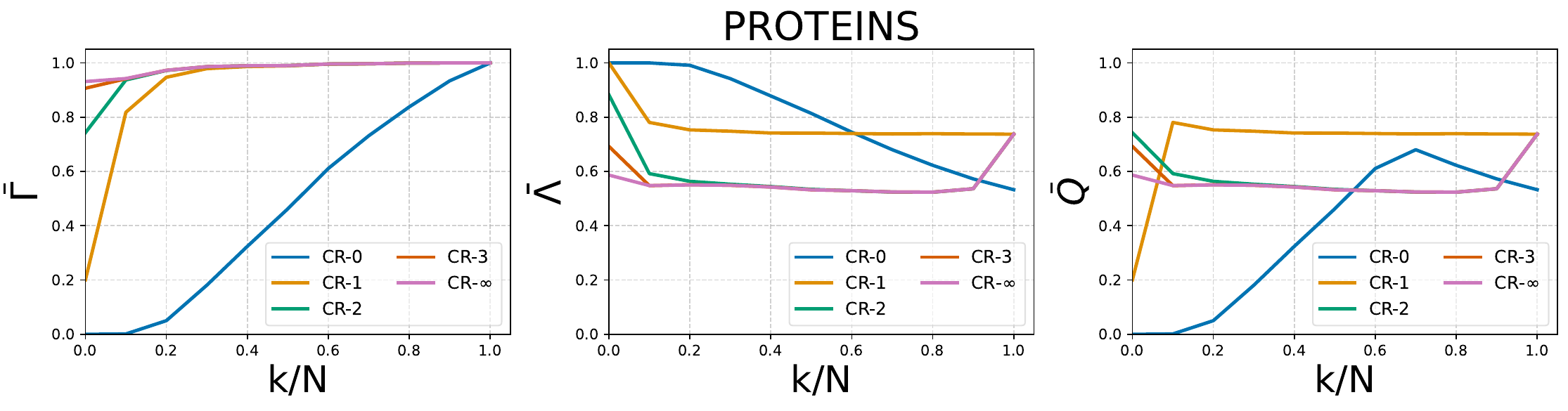}
\end{subfigure}

\begin{subfigure}{\textwidth}
    \centering
    \includegraphics[width=0.91\textwidth]{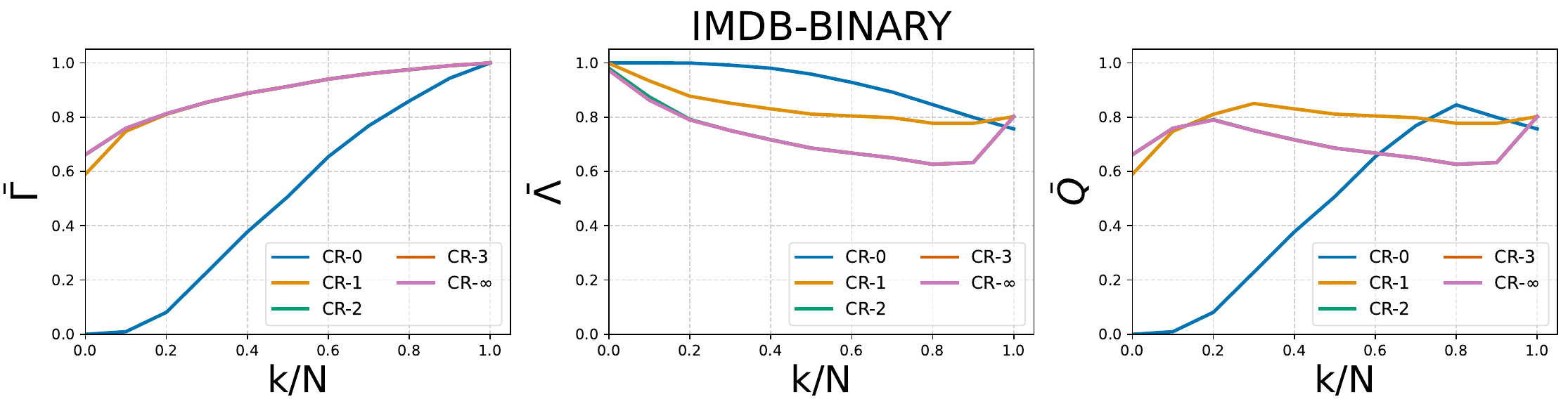}
\end{subfigure}

\begin{subfigure}{\textwidth}
    \centering
    \includegraphics[width=0.91\textwidth]{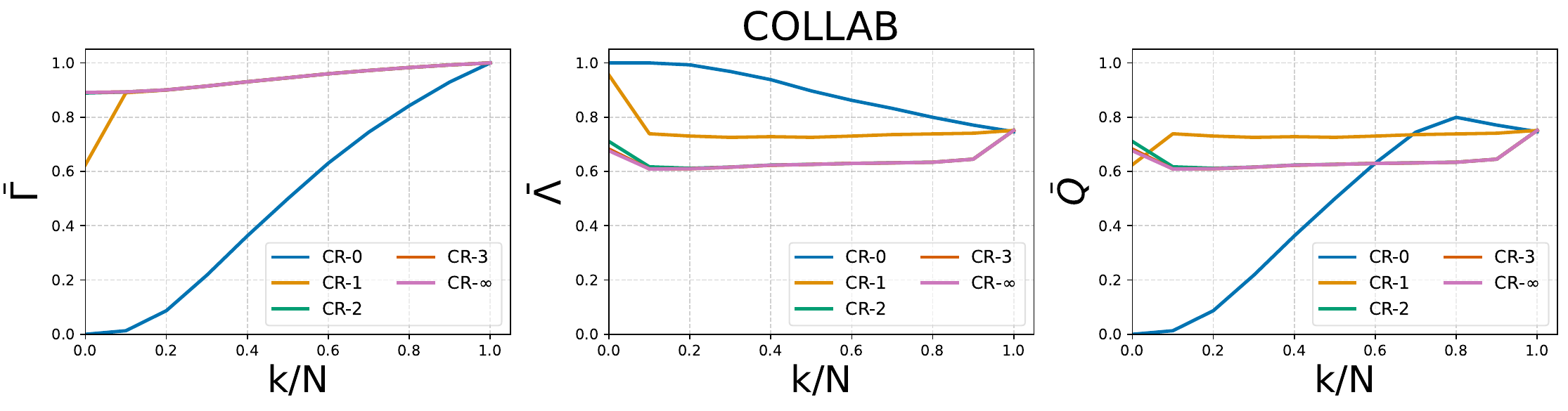}
\end{subfigure}
\caption{Group variant; counting the number of individual group matches without considering graph belonging.}
\label{fig:random_coloring_v3}
\end{figure*}

\clearpage
\section{Quality of colourings for additional datasets}
\label{sec:quality}

\begin{table*}[!th]
    \centering
    \resizebox{\linewidth}{!}{
    \begin{tabular}{l|ccc|ccc|ccc|ccc|ccc}
        \toprule
          & \multicolumn{3}{c|}{CR-0} & \multicolumn{3}{c|}{CR-1} &  \multicolumn{3}{c|}{CR-2} &  \multicolumn{3}{c|}{CR-3} & \multicolumn{3}{c}{CR-$\infty$} \\
           & $\overline{Q}$ & $\tau$ & $\overline{k/N}$  & $\overline{Q}$ & $\tau$ & $\overline{k/N}$ & $\overline{Q}$ & $\tau$ & $\overline{k/N}$ & $\overline{Q}$ & $\tau$ & $\overline{k/N}$ & $\overline{Q}$ &$\tau$ & $\overline{k/N}$\\
         \midrule
         Node Features & 0.03 & - & 0.19 & 0.03 & - & 0.19 & 0.00 & - & 0.19 & 0.00 & - & 0.19 & 0.00 & - & 0.19 \\
         \midrule
         Same      &  0.00 &  - &  0.04 &  0.21 &  - &  0.04 &  0.02 &  - &  0.04 &  0.00 &  - &  0.04 &  0.00 &  - &  0.04  \\
         Mixed     &  0.13 &  - &  0.51 &  0.02 &  - &  0.51 &  0.00 &  - &  0.51 &  0.00 &  - &  0.51 &  0.00 &  - &  0.51  \\
         Distinct  &  0.01 &  - &  1.00 &  0.03 &  - &  1.00 &  0.03 &  - &  1.00 &  0.03 &  - &  1.00 &  0.03 &  - &  1.00  \\
         \midrule
         Random Walk PE   &  \textbf{0.98} &  1.00 &  0.90 &  0.29 &  0.99 &  0.05 &  0.55 &  0.01 &  0.04 &  0.14 &  0.99 &  0.05 &  0.18 &  0.01 &  0.04 \\
         Laplacian PE     &  0.69 &  0.64 &  0.37 &  0.54 &  0.40 &  0.16 &  0.02 &  0.00 &  0.04 &  0.01 &  0.12 &  0.06 &  0.01 &  1.00 &  0.94 \\
         GPSE             &  0.50 &  0.98 &  0.26 &  0.44 &  0.96 &  0.12 &  0.03 &  1.00 &  1.00 &  0.03 &  1.00 &  1.00 &  0.03 &  1.00 &  1.00 \\
         Node2Vec PE      &  0.82 &  0.81 &  0.36 &  0.53 &  0.65 &  0.12 &  0.03 &  0.36 &  0.04 &  0.03 &  1.00 &  1.00 &  0.03 &  1.00 &  1.00 \\
         \bottomrule
    \end{tabular}
    }
    \caption{\textsc{ENZYMES}}
    \label{tab:qrdata_enzymes}
\end{table*}

\begin{table*}[!th]
    \centering
    \resizebox{\linewidth}{!}{
    \begin{tabular}{l|ccc|ccc|ccc|ccc|ccc}
        \toprule
          & \multicolumn{3}{c|}{CR-0} & \multicolumn{3}{c|}{CR-1} &  \multicolumn{3}{c|}{CR-2} &  \multicolumn{3}{c|}{CR-3} & \multicolumn{3}{c}{CR-$\infty$} \\
           & $\overline{Q}$ & $\tau$ & $\overline{k/N}$  & $\overline{Q}$ & $\tau$ & $\overline{k/N}$ & $\overline{Q}$ & $\tau$ & $\overline{k/N}$ & $\overline{Q}$ & $\tau$ & $\overline{k/N}$ & $\overline{Q}$ &$\tau$ & $\overline{k/N}$\\
         \midrule
         Node Features & 0.04 & - & 0.19 & 0.58 & - & 0.19 & 0.16 & - & 0.19 & 0.03 & - & 0.19 & 0.01 & - & 0.19     \\
         \midrule
         Same      &  0.00 &  - &  0.05 &  0.18 &  - &  0.05 &  0.34 &  - &  0.05 &  0.06 &  - &  0.05 &  0.02 &  - &  0.05  \\
         Mixed     &  0.31 &  - &  0.51 &  0.32 &  - &  0.51 &  0.01 &  - &  0.51 &  0.02 &  - &  0.51 &  0.02 &  - &  0.51  \\
         Distinct  &  0.01 &  - &  1.00 &  0.31 &  - &  1.00 &  0.31 &  - &  1.00 &  0.31 &  - &  1.00 &  0.31 &  - &  1.00  \\
         \midrule
         Random Walk PE   &  0.37 &  1.00 &  0.88 &  0.36 &  1.00 &  0.88 &  0.35 &  0.98 &  0.04 &  0.22 &  1.00 &  0.88 &  0.22 &  1.00 &  0.88 \\
         Laplacian PE     &  0.49 &  0.50 &  0.32 &  0.53 &  0.39 &  0.23 &  0.31 &  0.00 &  0.07 &  0.13 &  1.00 &  0.94 &  0.13 &  1.00 &  0.94 \\
         GPSE             &  0.49 &  0.99 &  0.50 &  0.61 &  0.98 &  0.25 &  0.36 &  0.91 &  0.06 &  0.31 &  1.00 &  1.00 &  0.31 &  1.00 &  1.00 \\
         Node2Vec PE      &  \textbf{0.65} &  0.76 &  0.27 &  0.59 &  0.67 &  0.16 &  0.34 &  0.00 &  0.04 &  0.31 &  1.00 &  1.00 &  0.31 &  1.00 &  1.00 \\
         \bottomrule
    \end{tabular}
    }
    \caption{\textsc{PROTEINS}}
    \label{tab:qrdata_proteins}
\end{table*}

\begin{table*}[!th]
    \centering
    \resizebox{\linewidth}{!}{
    \begin{tabular}{l|ccc|ccc|ccc|ccc|ccc}
        \toprule
          & \multicolumn{3}{c|}{CR-0} & \multicolumn{3}{c|}{CR-1} &  \multicolumn{3}{c|}{CR-2} &  \multicolumn{3}{c|}{CR-3} & \multicolumn{3}{c}{CR-$\infty$} \\
           & $\overline{Q}$ & $\tau$ & $\overline{k/N}$  & $\overline{Q}$ & $\tau$ & $\overline{k/N}$ & $\overline{Q}$ & $\tau$ & $\overline{k/N}$ & $\overline{Q}$ & $\tau$ & $\overline{k/N}$ & $\overline{Q}$ &$\tau$ & $\overline{k/N}$\\
         \midrule
         Node Features & 0.00 & - & 0.19 & 0.59 & - & 0.19 & 0.66 & - & 0.19 & 0.66 & - & 0.19 & 0.66 & - & 0.19     \\
         \midrule
         Same      &  0.00 &  - &  0.06 &  0.59 &  - &  0.06 &  0.66 &  - &  0.06 &  0.66 &  - &  0.06 &  0.66 &  - &  0.06  \\
         Mixed     &  0.51 &  - &  0.51 &  0.23 &  - &  0.51 &  0.09 &  - &  0.51 &  0.09 &  - &  0.51 &  0.09 &  - &  0.51  \\
         Distinct  &  0.01 &  - &  1.00 &  0.19 &  - &  1.00 &  0.19 &  - &  1.00 &  0.19 &  - &  1.00 &  0.19 &  - &  1.00  \\
         \midrule
         Random Walk PE   &  0.68 &  1.00 &  0.40 &  0.67 &  1.00 &  0.40 &  0.68 &  1.00 &  0.40 &  0.68 &  1.00 &  0.40 &  0.68 &  1.00 &  0.40 \\
         Laplacian PE     &  0.87 &  0.50 &  0.53 &  0.84 &  0.18 &  0.35 &  0.77 &  0.04 &  0.20 &  0.78 &  0.04 &  0.20 &  0.78 &  0.04 &  0.20 \\
         GPSE             &  0.54 &  0.99 &  0.18 &  0.64 &  0.99 &  0.18 &  0.67 &  0.99 &  0.18 &  0.67 &  0.99 &  0.18 &  0.67 &  0.99 &  0.18 \\
         Node2Vec PE      &  \textbf{0.88} &  0.71 &  0.45 &  0.88 &  0.59 &  0.34 &  0.84 &  0.44 &  0.24 &  0.84 &  0.44 &  0.24 &  0.84 &  0.44 &  0.24 \\
         \bottomrule
    \end{tabular}
    }
    \caption{\textsc{IMDB-BINARY}}
    \label{tab:qrdata_imdb}
\end{table*}

\begin{table*}[!th]
    \centering
    \resizebox{\linewidth}{!}{
    \begin{tabular}{l|ccc|ccc|ccc|ccc|ccc}
        \toprule
          & \multicolumn{3}{c|}{CR-0} & \multicolumn{3}{c|}{CR-1} &  \multicolumn{3}{c|}{CR-2} &  \multicolumn{3}{c|}{CR-3} & \multicolumn{3}{c}{CR-$\infty$} \\
           & $\overline{Q}$ & $\tau$ & $\overline{k/N}$  & $\overline{Q}$ & $\tau$ & $\overline{k/N}$ & $\overline{Q}$ & $\tau$ & $\overline{k/N}$ & $\overline{Q}$ & $\tau$ & $\overline{k/N}$ & $\overline{Q}$ &$\tau$ & $\overline{k/N}$\\
         \midrule
         Node Features & 0.00 & - & 0.19 & 0.62 & - & 0.19 & 0.17 & - & 0.19 & 0.17 & - & 0.19 & 0.17 & - & 0.19     \\
         \midrule
         Same      &  0.00 &  - &  0.02 &  0.62 &  - &  0.02 &  0.17 &  - &  0.02 &  0.17 &  - &  0.02 &  0.17 &  - &  0.02  \\
         Mixed     &  0.47 &  - &  0.50 &  0.03 &  - &  0.50 &  0.01 &  - &  0.50 &  0.01 &  - &  0.50 &  0.01 &  - &  0.50  \\
         Distinct  &  0.00 &  - &  1.00 &  0.10 &  - &  1.00 &  0.10 &  - &  1.00 &  0.10 &  - &  1.00 &  0.10 &  - &  1.00  \\
         \midrule
         Random Walk PE   &  0.41 &  1.00 &  0.62 &  0.66 &  0.98 &  0.03 &  0.40 &  1.00 &  0.62 &  0.40 &  1.00 &  0.62 &  0.40 &  1.00 &  0.62 \\
         Laplacian PE     &  0.71 &  0.65 &  0.21 &  0.71 &  0.28 &  0.05 &  0.21 &  0.98 &  0.55 &  0.21 &  0.98 &  0.55 &  0.21 &  0.98 &  0.55 \\
         GPSE             &  0.52 &  0.99 &  0.10 &  0.65 &  0.95 &  0.04 &  0.18 &  0.93 &  0.03 &  0.18 &  0.99 &  0.10 &  0.18 &  0.99 &  0.10 \\
         Node2Vec PE      &  \textbf{0.78} &  0.77 &  0.26 &  0.70 &  0.44 &  0.04 &  0.18 &  0.04 &  0.02 &  0.17 &  0.00 &  0.02 &  0.17 &  0.00 &  0.02 \\
         \bottomrule
    \end{tabular}
    }
    \caption{\textsc{COLLAB}}
    \label{tab:qrdata_collab}
\end{table*}

\newpage
\section{Experimental Setup and Reproducibility}
\label{sec:setup}
For our experiments, we compare different pooling methods under a unified experimental setup. 
We use standard datasets and models provided by PyG and apply the same set of hyperparameters across all models, including a learning rate of $1e-3$, the AdamW optimiser, a dropout rate of 0.5, 200 hidden channels, and a PE dimension of 6. 
We train for 15000 epochs with early stopping and patience of 50.
We fix an $80/10/10$ split and repeat each experiment 5 times with different seeds.
The used datasets are released as part of the \textit{TUDataset} collection by \cite{Morris+2020}.
We use the \textsc{Mutag}, \textsc{PROTEINS}, \textsc{ENZYMES}, \textsc{COLLAB}, and \textsc{IMDB-BINARY} datasets in our analysis and double the list of datasets for the main study by also including \textsc{COLORS}~\cite{castellana2025bnpoolbayesiannonparametricapproach}, \textsc{DD}, \textsc{Mutagenicity}, \textsc{NCI1}, and \textsc{REDDIT-BINARY}.
\Cref{alg:setup} presents the model procedure and our proposed enhancement.

Our implementation is available on GitHub at \url{https://github.com/jvpichowski-research/2026-features-for-graph-pooling}. 
All datasets and positional encodings are obtained with PyG.
We report standard deviations in \Cref{tab:main} and discuss influencing factors in the main text in \Cref{sec:guide}. 

Experiments are executed on an Nvidia L40.
The training and inference runtimes are comparable to the reference implementations provided by PyG, as our modifications do not introduce additional computational complexity to the models.

\begin{figure}[h!]
\centering
\begin{minipage}{\textwidth}
\begin{algorithm}[H]
\caption{Pooling GNN - Forward Pass}\label{alg:setup}
\begin{algorithmic}
\REQUIRE $X, A, \textsc{Objective}_{\textsc{pool}}$
\ENSURE $X_{\mathcal{G}}, \mathcal{L}_{\textsc{aux}}$
\STATE \textcolor{BrickRed}{$\rho \gets \textsc{Laplacian}(A)$}
\STATE \textcolor{gray}{$X' \gets \sigma_{\textsc{ReLU}}(\textsc{Conv}_{\Theta_1}(X, A))$}
\STATE \textcolor{BrickRed}{$H \gets [\textsc{Detach}(X), \rho]$}
\STATE $S \gets \sigma_{\textsc{Softmax}}(HW_{\Theta_p})$ \hspace{4em}\COMMENT{ SEL }
\STATE $X_p \gets S^TX'$ \hspace{8.3em}\COMMENT{ RED }
\STATE $A_p \gets S^TAS$ \hspace{8.2em}\COMMENT{ CON }
\STATE \textcolor{gray}{$X'_p \gets \sigma_{\textsc{ReLU}}(\textsc{Conv}_{\Theta_2}(X_p, A_p))$}
\STATE \textcolor{gray}{$X_{\mathcal{G}} \gets \textsc{Readout}_{\textsc{max}}(X_p')$}
\STATE $\mathcal{L}_{\textsc{aux}} \gets \textsc{Objective}_{\textsc{pool}}(S, A)$
\end{algorithmic}
\end{algorithm}
\caption*{\Cref{alg:setup}: Adapted pooling procedure. \textcolor{gray}{Gray:} The base (\textit{nopool}) GNN for graph classification consists of two convolution layers and a max read-out function. Black: The SRC-pooling framework adds the SEL, RED and CON operations based on a linear layer together with a pooling objective. \textcolor{BrickRed}{Red:} Our adaptation feeds the Laplacian PE with concatenation directly into the SEL operation. For efficiency, the Laplacian PE is calculated once in advance and reused in every pass. The intermediate embeddings are detached such that the correctly working pooling does not influence the downstream task GNN.}
\end{minipage}
\end{figure}

\newpage

\section{Detailed Results}
\label{sec:main-tab}

\input{main-tab}


\end{document}

%% file: main-tab.tex
\begin{table*}[!htb]
\centering
\resizebox{0.9\textwidth}{!}{%
\small
\setlength{\tabcolsep}{3pt}
\begin{tabular}{l|c@{}H:c@{}H:c@{}H:c@{}H:c@{}H}
\toprule
Operator 
& \multicolumn{2}{c|}{COLLAB}
& \multicolumn{2}{c|}{COLORS-3}
& \multicolumn{2}{c|}{DD}
& \multicolumn{2}{c|}{ENZYMES}
& \multicolumn{2}{c}{IMDB-BINARY} \\
\cmidrule(lr){2-3}\cmidrule(lr){4-5}\cmidrule(lr){6-7}\cmidrule(lr){8-9}\cmidrule(lr){10-11}
 & ACC & LOSS & ACC & LOSS & ACC & LOSS & ACC & LOSS & ACC & LOSS \\
\midrule
diffpool (laplacian) & $51.36 \pm 7.69$ & $0.01 \pm 0.01$ & $73.78 \pm 14.84$ & $0.41 \pm 0.23$ & $59.32 \pm 0.00$ & $0.01 \pm 0.01$ & $17.29 \pm 5.00$ & $0.11 \pm 0.08$ & $50.80 \pm 1.79$ & $0.04 \pm 0.03$ \\
diffpool (node2vec) & $48.28 \pm 5.55$ & $0.04 \pm 0.05$ & $58.88 \pm 7.16$ & $0.27 \pm 0.05$ & $59.32 \pm 0.00$ & $0.04 \pm 0.03$ & ${19.32 \pm 4.01}$ & $0.05 \pm 0.06$ & ${51.20 \pm 2.68}$ & $0.04 \pm 0.04$ \\
diffpool (gpse) & $50.76 \pm 7.12$ & $0.04 \pm 0.05$ & $68.36 \pm 13.44$ & $0.49 \pm 0.25$ & $59.32 \pm 0.00$ & $0.04 \pm 0.07$ & $14.92 \pm 7.33$ & $0.09 \pm 0.05$ & $50.00 \pm 0.00$ & $0.05 \pm 0.07$ \\
diffpool (rw) & $45.92 \pm 0.27$ & $0.05 \pm 0.08$ & $72.83 \pm 5.40$ & $0.40 \pm 0.24$ & $59.32 \pm 0.00$ & $0.08 \pm 0.11$ & $12.88 \pm 5.44$ & $0.02 \pm 0.02$ & $50.00 \pm 0.00$ & $0.05 \pm 0.06$ \\
diffpool (none) & ${56.92 \pm 4.77}$ & $0.14 \pm 0.25$ & ${81.88 \pm 1.48}$ & $0.54 \pm 0.10$ & ${67.63 \pm 11.38}$ & $0.07 \pm 0.14$ & $8.47 \pm 0.00$ & $0.00 \pm 0.00$ & ${51.20 \pm 1.10}$ & $0.00 \pm 0.00$ \\
\hdashline
dmon (laplacian) & ${69.80 \pm 1.44}$ & $0.43 \pm 0.00$ & $93.82 \pm 0.87$ & $1.35 \pm 0.00$ & $77.97 \pm 1.59$ & $0.51 \pm 0.00$ & $43.73 \pm 3.26$ & $0.20 \pm 0.01$ & ${62.80 \pm 2.59}$ & $0.12 \pm 0.01$ \\
dmon (node2vec) & $67.08 \pm 0.58$ & $0.42 \pm 0.00$ & $93.09 \pm 0.79$ & $1.35 \pm 0.00$ & $76.61 \pm 1.14$ & $0.52 \pm 0.00$ & $42.37 \pm 3.20$ & $0.21 \pm 0.00$ & $52.80 \pm 7.43$ & $0.16 \pm 0.00$ \\
dmon (gpse) & $69.08 \pm 1.04$ & $0.43 \pm 0.00$ & $95.16 \pm 0.75$ & $1.35 \pm 0.00$ & $78.81 \pm 1.99$ & $0.51 \pm 0.00$ & $51.19 \pm 2.21$ & $0.22 \pm 0.00$ & $52.80 \pm 7.53$ & $0.16 \pm 0.00$ \\
dmon (rw) & $68.60 \pm 0.60$ & $0.43 \pm 0.00$ & ${95.21 \pm 0.55}$ & $1.36 \pm 0.00$ & ${79.83 \pm 1.63}$ & $0.51 \pm 0.00$ & ${53.90 \pm 6.93}$ & $0.24 \pm 0.00$ & $49.80 \pm 2.28$ & $0.16 \pm 0.00$ \\
dmon (none) & $65.16 \pm 0.67$ & $0.42 \pm 0.00$ & $94.90 \pm 0.48$ & $1.36 \pm 0.00$ & $78.81 \pm 1.34$ & $0.51 \pm 0.00$ & $50.85 \pm 6.23$ & $0.23 \pm 0.00$ & $52.00 \pm 7.38$ & $0.16 \pm 0.00$ \\
\hdashline
jbpool (laplacian) & $66.56 \pm 0.33$ & $-0.19 \pm 0.02$ & $91.62 \pm 1.54$ & $-0.10 \pm 0.01$ & ${78.47 \pm 1.29}$ & $-0.06 \pm 0.00$ & $31.86 \pm 6.39$ & $-0.20 \pm 0.02$ & $49.00 \pm 1.73$ & $-0.28 \pm 0.04$ \\
jbpool (node2vec) & $66.68 \pm 0.92$ & $-0.19 \pm 0.02$ & $89.76 \pm 1.20$ & $-0.13 \pm 0.01$ & $75.93 \pm 4.51$ & $-0.17 \pm 0.01$ & $30.85 \pm 1.42$ & $-0.37 \pm 0.05$ & $50.80 \pm 3.27$ & $-0.30 \pm 0.04$ \\
jbpool (gpse) & $66.04 \pm 2.03$ & $-0.12 \pm 0.00$ & $93.41 \pm 1.57$ & $-0.11 \pm 0.01$ & ${78.47 \pm 2.29}$ & $-0.05 \pm 0.00$ & $31.19 \pm 6.52$ & $-0.18 \pm 0.00$ & $52.60 \pm 5.13$ & $-0.22 \pm 0.01$ \\
jbpool (rw) & ${67.48 \pm 0.79}$ & $-0.11 \pm 0.00$ & $93.87 \pm 1.46$ & $-0.11 \pm 0.01$ & ${78.47 \pm 2.58}$ & $-0.06 \pm 0.00$ & $26.44 \pm 5.03$ & $-0.19 \pm 0.00$ & ${54.20 \pm 2.95}$ & $-0.21 \pm 0.01$ \\
jbpool (none) & $65.40 \pm 1.21$ & $-0.08 \pm 0.00$ & ${95.44 \pm 0.49}$ & $-0.11 \pm 0.00$ & $77.63 \pm 1.29$ & $-0.05 \pm 0.00$ & ${47.12 \pm 6.16}$ & $-0.23 \pm 0.01$ & $50.20 \pm 0.45$ & $-0.21 \pm 0.00$ \\
\hdashline
mdlpool (laplacian) & $45.80 \pm 0.00$ & $5.77 \pm 0.00$ & $39.11 \pm 15.58$ & $5.56 \pm 0.51$ & ${59.32 \pm 0.00}$ & $7.46 \pm 0.38$ & $15.59 \pm 4.70$ & $4.93 \pm 0.02$ & ${50.00 \pm 0.00}$ & $4.04 \pm 0.03$ \\
mdlpool (node2vec) & $45.80 \pm 0.00$ & $5.77 \pm 0.00$ & ${41.36 \pm 7.84}$ & $5.48 \pm 0.29$ & ${59.32 \pm 0.00}$ & $6.83 \pm 0.11$ & ${17.97 \pm 5.13}$ & $4.78 \pm 0.08$ & ${50.00 \pm 0.00}$ & $4.06 \pm 0.00$ \\
mdlpool (gpse) & $45.80 \pm 0.00$ & $5.77 \pm 0.00$ & $19.85 \pm 11.74$ & $5.81 \pm 0.43$ & ${59.32 \pm 0.00}$ & $7.83 \pm 0.02$ & $8.47 \pm 0.00$ & $5.12 \pm 0.21$ & ${50.00 \pm 0.00}$ & $4.06 \pm 0.00$ \\
mdlpool (rw) & $45.80 \pm 0.00$ & $5.77 \pm 0.00$ & $32.87 \pm 15.87$ & $5.56 \pm 0.25$ & ${59.32 \pm 0.00}$ & $7.86 \pm 0.06$ & $10.51 \pm 3.67$ & $4.94 \pm 0.02$ & ${50.00 \pm 0.00}$ & $4.06 \pm 0.00$ \\
mdlpool (none) & ${50.12 \pm 6.08}$ & $5.77 \pm 0.00$ & $8.21 \pm 0.00$ & $5.33 \pm 0.12$ & ${59.32 \pm 0.00}$ & $8.31 \pm 1.03$ & $8.47 \pm 0.00$ & $4.93 \pm 0.00$ & ${50.00 \pm 0.00}$ & $4.06 \pm 0.00$ \\
\hdashline
mincut (laplacian) & ${70.04 \pm 1.50}$ & $0.30 \pm 0.00$ & $94.30 \pm 0.65$ & $0.29 \pm 0.00$ & $78.31 \pm 1.29$ & $0.35 \pm 0.00$ & $54.24 \pm 2.94$ & $0.18 \pm 0.00$ & $50.40 \pm 6.99$ & $0.17 \pm 0.00$ \\
mincut (node2vec) & $67.88 \pm 1.08$ & $0.30 \pm 0.00$ & $94.00 \pm 0.82$ & $0.29 \pm 0.00$ & $75.42 \pm 1.80$ & $0.35 \pm 0.00$ & $41.02 \pm 4.58$ & $0.24 \pm 0.00$ & $48.40 \pm 1.52$ & $0.17 \pm 0.00$ \\
mincut (gpse) & $68.44 \pm 0.86$ & $0.29 \pm 0.00$ & $95.55 \pm 1.01$ & $0.29 \pm 0.00$ & $79.83 \pm 1.52$ & $0.35 \pm 0.00$ & $49.83 \pm 3.08$ & $0.23 \pm 0.00$ & $51.20 \pm 6.26$ & $0.17 \pm 0.00$ \\
mincut (rw) & $68.56 \pm 0.22$ & $0.30 \pm 0.00$ & $95.62 \pm 0.66$ & $0.26 \pm 0.00$ & ${80.00 \pm 1.54}$ & $0.35 \pm 0.00$ & ${55.25 \pm 4.42}$ & $0.24 \pm 0.00$ & $47.60 \pm 2.61$ & $0.17 \pm 0.00$ \\
mincut (none) & $66.48 \pm 0.18$ & $0.29 \pm 0.00$ & ${95.72 \pm 0.41}$ & $0.29 \pm 0.00$ & $78.98 \pm 1.84$ & $0.35 \pm 0.00$ & $53.22 \pm 6.41$ & $0.23 \pm 0.00$ & ${52.60 \pm 7.60}$ & $0.17 \pm 0.00$ \\
\hdashline
nopool & $67.04 \pm 0.48$ & $0.00 \pm 0.00$ & $35.31 \pm 1.88$ & $0.00 \pm 0.00$ & $77.63 \pm 2.04$ & $0.00 \pm 0.00$ & $47.46 \pm 2.68$ & $0.00 \pm 0.00$ & $61.00 \pm 3.54$ & $0.00 \pm 0.00$ \\
\bottomrule
\end{tabular}%
}

\vspace{0.5em}

\resizebox{0.88\textwidth}{!}{%
\small
\setlength{\tabcolsep}{3pt}
\begin{tabular}{l|c@{}H:c@{}H:c@{}H:c@{}H:c@{}H}
\toprule
Operator 
& \multicolumn{2}{c|}{MUTAG}
& \multicolumn{2}{c|}{Mutagenicity}
& \multicolumn{2}{c|}{NCI1}
& \multicolumn{2}{c|}{PROTEINS}
& \multicolumn{2}{c}{REDDIT-BINARY} \\
\cmidrule(lr){2-3}\cmidrule(lr){4-5}\cmidrule(lr){6-7}\cmidrule(lr){8-9}\cmidrule(lr){10-11}
 & ACC & LOSS & ACC & LOSS & ACC & LOSS & ACC & LOSS & ACC & LOSS \\
\midrule
diffpool (laplacian) & $92.63 \pm 4.71$ & $0.06 \pm 0.13$ & $73.44 \pm 11.06$ & $0.04 \pm 0.08$ & $63.02 \pm 6.34$ & $0.08 \pm 0.03$ & $66.27 \pm 2.36$ & $0.06 \pm 0.03$ & ${75.70 \pm 9.28}$ & $0.02 \pm 0.04$ \\
diffpool (node2vec) & $90.53 \pm 6.86$ & $0.08 \pm 0.08$ & ${79.95 \pm 1.23}$ & $0.01 \pm 0.00$ & ${71.37 \pm 0.80}$ & $0.01 \pm 0.00$ & ${67.06 \pm 4.25}$ & $0.02 \pm 0.01$ & $62.10 \pm 12.64$ & $0.11 \pm 0.07$ \\
diffpool (gpse) & $90.53 \pm 9.42$ & $0.05 \pm 0.08$ & $77.66 \pm 2.11$ & $0.03 \pm 0.05$ & $64.73 \pm 6.08$ & $0.11 \pm 0.07$ & $60.39 \pm 0.88$ & $0.03 \pm 0.02$ & $63.10 \pm 5.56$ & $0.11 \pm 0.15$ \\
diffpool (rw) & $92.63 \pm 4.71$ & $0.01 \pm 0.00$ & $78.17 \pm 0.56$ & $0.02 \pm 0.02$ & $63.80 \pm 3.57$ & $0.06 \pm 0.04$ & $62.35 \pm 2.36$ & $0.05 \pm 0.03$ & $72.10 \pm 11.36$ & $0.00 \pm 0.00$ \\
diffpool (none) & ${93.68 \pm 2.35}$ & $0.01 \pm 0.00$ & $54.99 \pm 4.61$ & $0.00 \pm 0.00$ & $68.88 \pm 2.55$ & $0.00 \pm 0.00$ & $66.86 \pm 9.69$ & $0.01 \pm 0.01$ & $64.40 \pm 6.86$ & $0.03 \pm 0.07$ \\
\hdashline
dmon (laplacian) & $89.47 \pm 3.72$ & $0.06 \pm 0.02$ & ${82.25 \pm 1.32}$ & $0.16 \pm 0.01$ & $74.73 \pm 2.02$ & $0.03 \pm 0.00$ & $72.35 \pm 2.89$ & $0.74 \pm 0.00$ & $85.60 \pm 0.74$ & $2.61 \pm 0.00$ \\
dmon (node2vec) & $88.42 \pm 4.40$ & $0.06 \pm 0.01$ & $80.80 \pm 1.30$ & $0.23 \pm 0.00$ & $73.46 \pm 1.28$ & $0.12 \pm 0.00$ & ${74.90 \pm 1.12}$ & $0.75 \pm 0.00$ & $81.30 \pm 1.15$ & $2.62 \pm 0.00$ \\
dmon (gpse) & ${91.58 \pm 2.88}$ & $0.08 \pm 0.00$ & $80.28 \pm 1.44$ & $0.27 \pm 0.00$ & $75.37 \pm 0.90$ & $0.17 \pm 0.00$ & $73.33 \pm 0.82$ & $0.76 \pm 0.00$ & ${92.10 \pm 1.60}$ & $2.61 \pm 0.00$ \\
dmon (rw) & ${91.58 \pm 2.88}$ & $0.08 \pm 0.00$ & $81.03 \pm 1.78$ & $0.27 \pm 0.00$ & ${76.05 \pm 0.87}$ & $0.16 \pm 0.00$ & $74.31 \pm 0.82$ & $0.76 \pm 0.00$ & $89.80 \pm 2.75$ & $2.61 \pm 0.00$ \\
dmon (none) & $89.47 \pm 3.72$ & $0.08 \pm 0.00$ & $81.78 \pm 0.73$ & $0.27 \pm 0.00$ & $75.02 \pm 1.34$ & $0.16 \pm 0.00$ & $72.94 \pm 1.12$ & $0.76 \pm 0.00$ & $82.20 \pm 0.84$ & $2.61 \pm 0.00$ \\
\hdashline
jbpool (laplacian) & $88.42 \pm 4.40$ & $-0.53 \pm 0.07$ & $78.74 \pm 1.23$ & $-0.25 \pm 0.08$ & $67.85 \pm 1.87$ & $-0.27 \pm 0.07$ & $71.37 \pm 1.28$ & $-0.16 \pm 0.04$ & $85.10 \pm 2.30$ & $-0.03 \pm 0.00$ \\
jbpool (node2vec) & $81.05 \pm 7.98$ & $-0.58 \pm 0.03$ & $81.64 \pm 1.61$ & $-0.45 \pm 0.01$ & $69.32 \pm 1.41$ & $-0.43 \pm 0.02$ & $73.73 \pm 2.72$ & $-0.26 \pm 0.02$ & $80.80 \pm 1.04$ & $-0.04 \pm 0.00$ \\
jbpool (gpse) & ${93.68 \pm 2.35}$ & $-0.32 \pm 0.04$ & $78.74 \pm 1.46$ & $-0.18 \pm 0.01$ & $68.05 \pm 2.46$ & $-0.19 \pm 0.02$ & $73.14 \pm 1.78$ & $-0.11 \pm 0.00$ & $89.10 \pm 1.78$ & $-0.03 \pm 0.00$ \\
jbpool (rw) & $92.63 \pm 2.88$ & $-0.37 \pm 0.04$ & $77.52 \pm 0.88$ & $-0.20 \pm 0.03$ & $69.76 \pm 0.75$ & $-0.21 \pm 0.05$ & ${75.88 \pm 2.03}$ & $-0.11 \pm 0.00$ & ${91.50 \pm 1.32}$ & $-0.03 \pm 0.00$ \\
jbpool (none) & $89.47 \pm 0.00$ & $-0.38 \pm 0.00$ & ${83.23 \pm 0.43}$ & $-0.29 \pm 0.00$ & ${72.24 \pm 1.61}$ & $-0.32 \pm 0.01$ & $72.55 \pm 1.55$ & $-0.11 \pm 0.01$ & $82.60 \pm 0.89$ & $-0.03 \pm 0.00$ \\
\hdashline
mdlpool (laplacian) & ${94.74 \pm 0.00}$ & $3.99 \pm 0.00$ & $74.85 \pm 5.64$ & $4.57 \pm 0.00$ & ${69.90 \pm 5.76}$ & $4.65 \pm 0.00$ & $64.12 \pm 4.08$ & $4.59 \pm 0.42$ & $47.10 \pm 2.75$ & $6.52 \pm 0.00$ \\
mdlpool (node2vec) & ${94.74 \pm 0.00}$ & $3.99 \pm 0.00$ & ${79.39 \pm 1.91}$ & $4.59 \pm 0.02$ & $68.05 \pm 9.06$ & $4.57 \pm 0.06$ & ${68.24 \pm 3.83}$ & $4.65 \pm 0.03$ & ${51.20 \pm 9.07}$ & $6.56 \pm 0.07$ \\
mdlpool (gpse) & ${94.74 \pm 0.00}$ & $3.99 \pm 0.00$ & $75.27 \pm 1.73$ & $4.57 \pm 0.00$ & $69.66 \pm 2.38$ & $4.65 \pm 0.00$ & $62.35 \pm 2.36$ & $4.77 \pm 0.01$ & $50.90 \pm 8.58$ & $6.58 \pm 0.14$ \\
mdlpool (rw) & ${94.74 \pm 0.00}$ & $3.99 \pm 0.00$ & $73.72 \pm 11.25$ & $4.57 \pm 0.00$ & $61.80 \pm 6.41$ & $4.65 \pm 0.00$ & $60.98 \pm 2.63$ & $4.85 \pm 0.20$ & $50.80 \pm 4.72$ & $6.52 \pm 0.00$ \\
mdlpool (none) & ${94.74 \pm 0.00}$ & $3.99 \pm 0.00$ & $57.33 \pm 6.04$ & $4.57 \pm 0.00$ & $60.00 \pm 9.62$ & $4.65 \pm 0.00$ & $63.33 \pm 7.89$ & $4.76 \pm 0.00$ & $46.60 \pm 2.53$ & $8.03 \pm 1.03$ \\
\hdashline
mincut (laplacian) & ${92.63 \pm 2.88}$ & $0.14 \pm 0.04$ & ${81.87 \pm 1.26}$ & $0.17 \pm 0.02$ & $75.12 \pm 0.71$ & $0.12 \pm 0.04$ & $70.98 \pm 2.03$ & $0.21 \pm 0.00$ & $86.00 \pm 1.17$ & $0.37 \pm 0.00$ \\
mincut (node2vec) & $85.26 \pm 11.41$ & $0.16 \pm 0.00$ & $80.75 \pm 0.38$ & $0.22 \pm 0.00$ & $74.15 \pm 1.22$ & $0.22 \pm 0.00$ & $74.12 \pm 0.88$ & $0.25 \pm 0.00$ & $87.10 \pm 1.88$ & $0.37 \pm 0.00$ \\
mincut (gpse) & $90.53 \pm 2.35$ & $0.16 \pm 0.00$ & $80.98 \pm 0.80$ & $0.22 \pm 0.00$ & ${76.44 \pm 1.25}$ & $0.22 \pm 0.00$ & $72.16 \pm 1.12$ & $0.25 \pm 0.00$ & ${92.60 \pm 0.96}$ & $0.37 \pm 0.00$ \\
mincut (rw) & $91.58 \pm 2.88$ & $0.16 \pm 0.00$ & $81.55 \pm 0.80$ & $0.22 \pm 0.00$ & $75.66 \pm 0.50$ & $0.21 \pm 0.00$ & ${75.69 \pm 1.28}$ & $0.25 \pm 0.00$ & $88.50 \pm 1.77$ & $0.37 \pm 0.00$ \\
mincut (none) & $87.37 \pm 2.88$ & $0.13 \pm 0.00$ & $80.98 \pm 0.48$ & $0.22 \pm 0.00$ & $75.27 \pm 0.92$ & $0.22 \pm 0.00$ & $72.55 \pm 0.69$ & $0.25 \pm 0.00$ & $86.00 \pm 1.90$ & $0.37 \pm 0.00$ \\
\hdashline
nopool & $78.95 \pm 13.42$ & $0.00 \pm 0.00$ & ${83.33 \pm 0.69}$ & $0.00 \pm 0.00$ & $73.41 \pm 0.88$ & $0.00 \pm 0.00$ & ${78.24 \pm 2.24}$ & $0.00 \pm 0.00$ & $73.30 \pm 1.35$ & $0.00 \pm 0.00$ \\
\bottomrule
\end{tabular}%
}

\caption{Performance comparison across all methods and datasets. \Cref{sec:setup} provides implementation details and experiment setup.}
\label{tab:main}

\end{table*}